\documentclass[12pt]{article}
\usepackage{amssymb,amsthm,amsmath,amsfonts,caption,slashed,bbm,latexsym,xcolor}
\usepackage{eqparbox}
\usepackage[caption=false]{subfig}
\usepackage{dsfont}
\usepackage{tcolorbox}
\usepackage{float}
\allowdisplaybreaks
\newcommand{\ft}[2]{{\textstyle\frac{#1}{#2}}}

\newcommand{\bbox}{\lower.2ex\hbox{$\Box$}}
\newsavebox{\uuunit}

\csname @addtoreset\endcsname{equation}{section}

   \usepackage[nosort]{cite}
   \pdfoutput=1
  \usepackage[pdftex]{hyperref}
\newcommand{\dr}{\raise.3ex\hbox{$\stackrel{\leftarrow}{\delta  }$}{}}
\newcommand{\dl}{\raise.3ex\hbox{$\stackrel{\rightarrow}{\delta }$}{} }
\newcommand{\pl}{\raise.3ex\hbox{$\stackrel{\rightarrow}{\partial }$}{} }
\newcommand*{\Z}{\mathbb{Z}}
\newcommand*{\Q}{\mathbb{Q}}
\usepackage[
  style=long3colheader,
  hyperfirst=false
  ]{glossaries}

\newcommand{\ignore}[1]{}

\newglossaryentry{integers}{name=\ensuremath{\Z},
  description={the ring of integers}, sort=!}
\newglossaryentry{rationals}{name=\ensuremath{\Q},
  description={the field of rational numbers}, sort=!, nonumberlist}

\newglossaryentry{vector-space}{name=\ensuremath{V},
  description={a vector space}, sort=V}

\makeglossaries

\usepackage{kbordermatrix}
\renewcommand{\kbldelim}{(}
\renewcommand{\kbrdelim}{)}

\newcommand\bovermat[2]{%
  \makebox[0pt][l]{$\smash{\overbrace{\phantom{%
    \begin{matrix}#2\end{matrix}}}^{\text{#1}}}$}#2}
\newcommand\partialphantom{\vphantom{\frac{\partial e_{P,M}}{\partial w_{1,1}}}}
\newtheorem{theorem}{Theorem}
\newtheorem{corollary}{Corollary}[theorem]
\newtheorem{lemma}[theorem]{Lemma}
\newtheorem{exmp}{Example}

\theoremstyle{definition}
\newtheorem{definition}{Definition}

\theoremstyle{remark}
\newtheorem*{remark}{Remark}

\usepackage[capitalize,nameinlink]{cleveref}

\begin{document}

\begin{titlepage}
\vspace{.5cm}
\begin{center}
\baselineskip=16pt
{\LARGE  A Step Towards Uncovering The Structure of Multistable Neural Networks  
}\\
\vfill
{\large  {\bf Magnus Tournoy}$^{1,2}$ and {\bf Brent Doiron}$^{1,2}$} \\
\vfill

{\small$^1$  Departments of Neurobiology and Statistics, University of Chicago, Chicago, IL, USA\\\smallskip
$^2$ Grossman Center for Quantitative Biology and Human Behavior, University of Chicago, Chicago, IL,
USA \\[2mm] }
\end{center}
\vfill
\begin{center}
{\bf Abstract}
\end{center}
{\small
We study how the connectivity within a recurrent neural network determines and is determined by the multistable solutions of network activity. To gain analytic tractability we let neural activation be a non-smooth Heaviside step function. This nonlinearity partitions the phase space into regions with different, yet linear dynamics. In each region either a stable equilibrium state exists, or network activity flows to outside of the region. The stable states are identified by their semipositivity constraints on the synaptic weight matrix. The restrictions can be separated by their effects on the signs or the strengths of the connections. Exact results on network topology, sign stability, weight matrix factorization, pattern completion and pattern coupling are derived and proven. Our work may lay the foundation for multistability in more complex recurrent neural networks. 
} \vfill

\hrule width 3.cm
{\footnotesize \noindent e-mails: tournoy@uchicago.edu,  bdoiron@uchicago.edu}
\end{titlepage}

\addtocounter{page}{1}

\printglossaries
\section{Introduction}
With experimental advances in the simultaneous recording of large populations of neurons \cite{stevenson2011advances,urai2022large}, the mathematical understanding of high dimensional nonlinear neural networks is of increasing interest \cite{vogels2005neural,yang2020artificial}. Ideally one would like to be able to relate the dynamical to the structural properties of the network and vice versa. However these relations are easily obscured by the complexity present in many biologically realistic network models. Models that are tractable, simple and yet rich in their dynamical scope offer frameworks in which the structure to function interdependence can be exactly formulated. They allow us to understand the limits/possibilities of network dynamics in more general systems.

In this article we focus on the existence of stable equilibrium solutions and the associated constraints on the structure of network connectivity. The presence of multiple stable equilibria, which can be thought of as stored/memorized patterns, critically depends on the nonlinear activation function which maps neuronal inputs to outputs \cite{grossberg1988nonlinear,hirsch1989convergent,
zhang2014comprehensive,cheng2006multistability}. In our recurrent neural circuit model we simplify the analysis by setting the neural activation function to be a Heaviside step function. In this infinite gain limit, the continuous-time Hopfield model \cite{hopfield1984neurons} and related circuit models \cite{amari1972characteristics,sompolinsky1988chaos,harris2015bifurcations} become what are known as Glass networks \cite{edwards2000analysis}. These were originally developed by Glass \cite{glass1978stable} to study Boolean networks with discrete switching dynamics, and in later years have been used to model both networks of neurons \cite{lewis1991steady,lewis1992nonlinear} and genes \cite{edwards2000combinatorial}. These studies have shown that, even though the nonlinear dynamics is restricted to \textit{switching manifolds} in phase space, the system exhibits complex dynamics such as steady states, limit cycles and chaos \cite{harris2015bifurcations,lewis1992nonlinear,filippov2013differential}. This suggests that the model offers a computational advantage by allowing us to study the nonlinear effects in a discrete and local manner without sacrificing dynamical richness/computational behavior. 

Notwithstanding the step function being a natural limit of the smooth sigmoidal activation present in many models, the nonsmooth activation functions that have been predominantly studied in the context of stability are linear with a rectification \cite{tang2007neural}, and hence lack any saturation. In this case, the conditions for multistability, global stability in terms of constraints on the symmetric weight matrix were derived by Hahnloser et al. \cite{hahnloser2000permitted}. These were expanded upon by many others, leading to exact results on the perturbative, topological and dynamical properties of threshold-linear networks (TLNs) \cite{wersing2001dynamical,yi2003multistability,tang2006dynamics,curto2012flexible,morrison2016diversity,curto2016pattern,zhang2018theoretical,curto2019fixed}. Recently the structure to function relation of the network was also explored by geometric analysis \cite{biswas2022geometric,curto2020combinatorial}.

Nevertheless, Glass networks are a fundamental class of threshold activated networks, in the sense that all the non-linear switching dynamics of the model class is fully captured by the Glass network. They are the simplest choice to study the nonlinear properties of continuous-time neural network dynamics while keeping the benefit of local linear behaviour. 

In our work we study the impact of the network connectivity on the multistable character of network solutions in Glass networks. Apart from a non-vanishing output constraint, no a priori assumptions are made on the connection weights of the network. Our attempt is to have a most general discussion.\\

The article is organized as follows: \\

In Section \ref{sec:def} we define the Glass network model and relevant mathematical objects. Next are the main theorems which are ordered into three different parts. \\

In Section \ref{sec:ss} we are concerned with the existence of stable steady states and the associated constraints on the weight matrix. We show that the presence of stable states imposes semipositivity constraints on the synaptic weight matrix. The consequences hereof can be divided into those that result from restrictions on the signs of the connections, i.e. the configurations/topology of the network and those that restrict the weights, i.e. the competition vs. cooperation between neurons. \\

In Section \ref{sec:sp} we focus on the consequences of the stable state condition on the configurations of the network. The matrix sign pattern classes that are necessary or sufficient for the existence of stable states are given and their consequences for network configurations are derived. The existence of sign stable states, i.e. states which are required to be stable by the sign pattern, exemplifies how neural networks can achieve stability independent from synaptic strengths and hence solely by their topology. Within the context of our network we proof that sign stable states are always minimally stable, i.e. have no stable substates.\\

In Section \ref{sec:fdi} we cover three distinct algebraic properties of Glass networks with multiple stable steady states: weight matrix factorization, stable state (de)composition and stable state coupling. 
\begin{itemize}
\item In Subsection \ref{subs:fac} we formulate how weight matrices of networks with stable states are identified by a unique semipositive matrix factorization. This factorization is a consequence of the geometrical properties of semipositive matrices, i.e. they are maps between proper polyhedral cones \cite{tsatsomeros2016geometric}.
\item In Subsection \ref{subs:dc} we give the necessary and sufficient conditions the connection strengths must satisfy in order for stable states to be (de)composable into stable (micro)macrostates. These results are of importance for the pattern completing capabilities of the network.
\item In Subsection \ref{subs:int} we show that the (de-)composition theorems turn out to be
derivable from a more general state coupling theorem that is a consequence of the Boolean logical structure of the Glass network. The theorem is of importance for the storage capacities of the network.
\end{itemize}

In Section \ref{sec:con} we end by giving a conclusion.

\section{Glass Networks} \label{sec:def}
We define the nonlinear neural network\footnote{In Appendix \ref{app:nioe} the equivalence of nonlinear inputs $\theta\left(x^i\right)$ and nonlinear outputs $\theta\left(W^i{}_j\, x^j\right)$ in the context of our work is explained.}
\begin{equation}
\dot x^i = -x^i+W^i{}_j\,\theta\left(x^j\right)\,.
\label{model}
\end{equation}
The index $i=1,\ldots,n$ runs over neural units, e.g. individual neurons or neuronal assemblies, that are parametrized by the variables $x^i\in\mathbb{R}$. The synaptic weights $w_{ij}$ are quantified by the matrix operator $W^i{}_j$. Throughout the paper we use Einstein notation meaning that indices appearing both ``up" and ``down" imply a summation of the corresponding components. The Heaviside step function
\begin{equation}
\theta\left(x^i\right)=
\begin{cases}
&1\qquad \text{for } x^i>0 \\
&0 \qquad \text{for } x^i\leq 0
\end{cases}
\,.
\label{hsf}
\end{equation}
partitions the dynamics for every $i\in[n]$ into two regions. The whole phase space therefore becomes a set of $2^n$ dynamically distinct orthants. From a set theoretic perspective the Heaviside step function is selecting a subset $\alpha\subset [n]$ for which the units are ``active". We can hence use $\alpha$ as an \textit{index} to construct the \textit{binary codes} 
\begin{equation}
p_\alpha^i\equiv
\begin{cases}
&1\qquad \text{for } i\in\alpha \\
&0 \qquad \text{for } i\not\in \alpha
\end{cases}
\,.
\end{equation}
The \textit{parts} of the partition themselves can then be defined in terms of these codes
\begin{equation}
{\cal P}_{\boldsymbol{\alpha}}\equiv\left\{x^i\,|\,\theta(x^i)=p_\alpha^i\right\}\,.
\label{defP}
\end{equation}
The following two conditions will be imposed on the network:
\begin{itemize}
\item \textbf{Embedding.} It is always possible to include an external input $\boldsymbol{\mu}$ to the system by having it embedded as a feedforward unit in the network
\begin{equation}
W=
\begin{pmatrix}
w_{ij} & \mu_i\\
0 & 1
\end{pmatrix}\,.
\end{equation}
The dynamics takes place on the hyperplane $x_{n}=1$. One can therefore restrict the analysis to this subset of the phase space.  
\item \textbf{Constraint. }We will assume that at any moment in time, the network is producing some output, be it by synaptic or external activation. This requires that for
\begin{itemize}
\item Vanishing external input:
\begin{equation}
W^i{}_j\,\theta(x^j)\neq 0\qquad \forall \boldsymbol{x}\in \mathbb{R}_+^n\setminus \left\{\boldsymbol{0}\right\}\, \label{nvi}
\end{equation}
\item Nonvanishing external input:
\begin{equation}
W^i{}_j\,\theta(x^j)\neq 0\qquad \forall \boldsymbol{x}\in \left\{\mathbb{R}_+^n\,|\,x_{n}=1\right\}\,.\label{nvi2}
\end{equation}
\end{itemize}
The nonzero output of the units will drive the system away from the boundaries between the parts. This is where some of the units are silent. As a consequence stable states will lie within the interior of the parts. We still allow for the completely ``inactive" state, i.e. $\boldsymbol{p_\emptyset}=\boldsymbol{0}$, to settle in the origin in the case of vanishing external input. In the case of an embedded external current the constraint is restricted to the hyperplane $x_n=1$.
\end{itemize}
\begin{exmp} \label{exmp1}
Suppose we have a $2$-dim network with vanishing external input $\boldsymbol{\mu}=\boldsymbol{0}$ and the following values for the weight matrix 
\begin{equation}
W=\begin{pmatrix}
1 & 4 \\
2 & 3 
\end{pmatrix}\,.
\end{equation}
The quadrants of the $2$-dim phase space are identified by the four index sets
\begin{equation}
\emptyset,\,\left\{1\right\},\,\left\{2\right\},\, \left\{1,2\right\}\,.
\end{equation}
They are all accompanied by their respective codes
\begin{equation}
\boldsymbol{p}_\emptyset=\begin{pmatrix}0\\ 0\end{pmatrix},\,
\boldsymbol{p}_{\left\{1\right\}}=\begin{pmatrix}1\\ 0\end{pmatrix},\,
\boldsymbol{p}_{\left\{2\right\}}=\begin{pmatrix}0\\ 1\end{pmatrix},\,
\boldsymbol{p}_{\left\{1,2\right\}}=\begin{pmatrix}1\\ 1\end{pmatrix}\,.
\end{equation}
Since that
\begin{equation} \label{fop1}
W\boldsymbol{p}_{\left\{1\right\}}=\begin{pmatrix}1\\ 2\end{pmatrix},\,
W\boldsymbol{p}_{\left\{2\right\}}=\begin{pmatrix}4\\ 3\end{pmatrix},\,
W\boldsymbol{p}_{\left\{1,2\right\}}=\begin{pmatrix}5\\ 5\end{pmatrix}
\end{equation}
are all nonzero, the constraint in $\left(\ref{nvi}\right)$ is satisfied.
\end{exmp}
\begin{exmp} \label{exmp2}
Suppose we have a $3$-dim network with one of the units functioning as an external input
\begin{equation}\label{W2}
W=\begin{pmatrix}
2 & 0 & -1 \\
0 & 2 & -1 \\
0  & 0 & 1
\end{pmatrix}\,.
\end{equation}
Because the third unit is active but fixed by the external source, the dynamics is constrained to the orthants where $x_3>0$. These parts correspond to the index sets
\begin{equation} 
\left\{3\right\},\,
\left\{1,3\right\},\,
\left\{2,3\right\}\,,
\left\{1,2,3\right\}
\end{equation}
with codes
\begin{equation}
\boldsymbol{p}_{\left\{3\right\}}=\begin{pmatrix}0\\ 0\\1\end{pmatrix},\,
\boldsymbol{p}_{\left\{1,3\right\}}=\begin{pmatrix}1\\ 0\\1\end{pmatrix},\,
\boldsymbol{p}_{\left\{2,3\right\}}=\begin{pmatrix}0\\ 1\\1\end{pmatrix},\,
\boldsymbol{p}_{\left\{1,2,3\right\}}=\begin{pmatrix}1\\ 1\\1\end{pmatrix}\,.
\end{equation}
Again one can easily check that the constraint in $\left(\ref{nvi2}\right)$ is satisfied
\begin{equation} \label{fop2}
W\boldsymbol{p}_{\left\{3\right\}}=\begin{pmatrix}-1\\ -1\\1\end{pmatrix},\,
W\boldsymbol{p}_{\left\{1,3\right\}}=\begin{pmatrix}1\\ -1\\1\end{pmatrix},\,
W\boldsymbol{p}_{\left\{2,3\right\}}=\begin{pmatrix}-1\\ 1\\1\end{pmatrix},\,
W\boldsymbol{p}_{\left\{1,2,3\right\}}=\begin{pmatrix}1\\ 1\\1\end{pmatrix}\,.
\end{equation}
\end{exmp}
\section{Stable States of Glass Network Activity}  \label{sec:ss}
\subsection{Internal Dynamics}
The first step to understanding the dynamics of \eqref{model} is to restrict one's focus to a single part ${\cal P}_\alpha$. Since all nonlinearities are pushed onto the boundaries, i.e. the axes of the orthants, the dynamics is linear in the interior of ${\cal P}_\alpha$. This is one of the strengths of using the Glass network model; it allows us to study nonlinear dynamics in a framework where almost all of the phase space is linearly driven. Yet the part to part dynamics exhibits rich behaviour emerging from the boundary nonlinearities. Within the part the nonlinear step function for $\boldsymbol{x_{\alpha}}\in {\cal P}_{\alpha}$ reduces to
\begin{equation}
\theta\left(x_{\alpha}^i\right)=p_\alpha^i\,.
\end{equation}
The internal dynamics becomes
\begin{equation}
\dot x_{\alpha}^i = -x_{\alpha}^i+W_{\alpha}^i\,.
\label{partdyn}
\end{equation}
where $\boldsymbol{W_{\alpha}}$ is the \textit{attractor point} of the part\footnote{Also called \textit{focal point} \cite{lewis1991steady} and \textit{virtual fixed point} \cite{curto2020combinatorial}.}
\begin{align}
W^i_{\alpha}&\equiv W^i{}_j\,p_\alpha^j\nonumber\\
&=\sum_{\text{columns: } i\,\in\, \alpha} W\,.\label{fcp}
\end{align} 
It symbolizes the effective strength of the connections and determines the direction of the flow in ${\cal P}_{\alpha}$. This can be readily seen from the solutions to \eqref{partdyn}, which are
\begin{equation}
x_\alpha^i(t)=W_\alpha^i+\left(x_\alpha^i(0)-W_\alpha^i\right)e^{-t}\,.
\label{eqm}
\end{equation}
Within each orthant the curves follow straight lines directed to the attractor point. At the boundaries sharp transitions can occur. These nonlinear jumps give rise to a plethora of rich dynamics; numerical integration of the equations of motion displays steady state, limit cycle and chaotic dynamics \cite{lewis1992nonlinear}.
\subsection{Fixed Points in Network Activity}
For the analysis of the fixed points one has to clarify the distinction between fixed points of the network dynamics and the attractor points of the parts. While within each part the system linearly flows towards a single attractor point
\begin{equation}
x_\alpha^*{}^i=W_{\alpha}^i\,,
\label{fp}
\end{equation}
this is only a fixed point of the system-wide dynamics when the flow ends within the part, namely when $\boldsymbol{x^*_\alpha}\in{\cal P}_{\alpha}$.\footnote{Sets that contain their attractor point have been called \textit{admissible} \cite{curto2020combinatorial}.} From the definitions \eqref{defP} and \eqref{fp}, one sees that the condition to be satisfied is
\begin{equation}
\theta\left(W^i_{\alpha}\right)=p_\alpha^i\,.
\label{fpcond}
\end{equation}
By the definition \eqref{hsf} of the step function and the constraint $\left(\ref{nvi},\ref{nvi2}\right)$ we have
\begin{equation}
 W^i_{\boldsymbol{\alpha}}>0 \text{ for }i\in\alpha \quad\text{\&}\quad W^i_{\boldsymbol{\alpha}}< 0 \text{ for }i\not\in\alpha
\label{fpcond2}
\end{equation}
or rewritten using the binary codes
\begin{equation}
 p_\alpha^i W^i_{\boldsymbol{\alpha}}>0 \quad\text{\&}\quad p_{\alpha^c}^i W^i_{\boldsymbol{\alpha}}< 0\,.
\label{fpcond2b}
\end{equation}
where we used $\boldsymbol{p}_{\alpha^c}=\boldsymbol{1}-\boldsymbol{p_\alpha}$ for the complementing code. Intuitively a part contains its attractor point whenever its active units are, on average, internally cooperative/excitatory (first inequality in \eqref{fpcond2}) but externally competitive/inhibitory (second inequality in \eqref{fpcond2}).

\subsection{Stable Network States}
In our model all fixed points lie within the interior of the parts. This is a consequence of the fixed point equation \eqref{fp} and the constraint $\left(\ref{nvi},\ref{nvi2}\right)$. The stability of the fixed points is therefore completely determined by the linear internal dynamics. From \eqref{partdyn} one immediately sees that for small enough perturbations the system stays within the part and gets dragged back to the fixed point by the negative leak term. Therefore \textbf{\textit{in our model all fixed points are stable}}. In the Glass network the study of stable states is in one-to-one correspondence with uncovering the set of fixed points. 
\begin{remark}
In the rest of the article we will call a part with stable fixed point a \textit{stable set}. When unambiguous we simplify the expressions by using the set ${\cal P}_\alpha$ interchangeably with the index $\alpha$.
\end{remark}
\subsection{Stable Set Condition} \label{sss:fpc}
The inequalities in \eqref{fpcond2} guarantee a stable fixed point. It is hence important to better understand the structure they impose on the weight matrix. As a first step let's introduce the \textit{signature of the part}
\begin{equation}
s^i_\alpha=p_\alpha^i-p_{\alpha^c}^i\,.
\label{signvec}
\end{equation}
By subtracting the two independent inequalities in \eqref{fpcond} we are led to the condition
\begin{equation}
x_*^i\in{\cal P}_\alpha \quad \Leftrightarrow\quad s_\alpha^i W_\alpha^i > 0\,.
\label{fpcond3}
\end{equation}
Secondly, we can write \eqref{fpcond3} into a matrix expression by defining the diagonal matrix form of the signature vector and the code
\begin{equation}
S_\alpha=\text{diag}\left(\boldsymbol{s}_\alpha\right)\quad \text{\&}\quad P_\alpha=\text{diag}\left(\boldsymbol{p}_\alpha\right)\,.
\end{equation}
These matrices are respectably known as signature and projection matrices. The inequality in \eqref{fpcond3} produces the following condition on the weight matrix
\begin{center}
\begin{tcolorbox}[title=Stable Set Condition, width=\linewidth/2]
\begin{equation}
S_\alpha W P_\alpha\,\boldsymbol{1}> \boldsymbol{0}
\label{fpcond4}
\end{equation}
\end{tcolorbox}
\end{center}
The inequality is a semipositivity condition \cite{fiedler1966some,vandergraft1972applications,johnson_smith_tsatsomeros_2020}. A matrix $A$ is said to be semipositive whenever there exists a vector $\boldsymbol{v}\geq \boldsymbol{0}$ for which $A\boldsymbol{v}>\boldsymbol{0}$.\footnote{Analogously a matrix $A$ is said to be \textit{seminonnegative} whenever there exists a vector $\boldsymbol{v}\geq \boldsymbol{0}$, $\boldsymbol{v}\neq \boldsymbol{0}$ for which $A\boldsymbol{v}\geq \boldsymbol{0}$.} The condition \eqref{fpcond4} is even more strict since it specifies the \textit{semipositivity vector}, namely $S_\alpha W P_\alpha$ must be positive for the ``all-ones vector" $\boldsymbol{1}$ or equivalenty $S_\alpha W$ must be positive for the code $\boldsymbol{p_\alpha}$. The constraint reveals two complementary ways by which the structure of connections leads to multistable dynamics:
\begin{itemize}
\item \textbf{Signs of connections.} There are certain combinations of signs for the row elements that are disallowed by \eqref{fpcond4}. For example rows with only inhibitory connections are disallowed for units that are ``active". This is evident because they have to sustain their activation against the leak term in \eqref{model}.
\item \textbf{Strengths of connections.} Within each row there may be a combination of inhibitory and excitatory connections. The relative strength of these connections will determine whether a unit stays ``active" or will be silenced.
\end{itemize}
The signs of the connections are necessary for \eqref{fpcond4} to be satisfied and can in some cases be sufficient. The strength of the connections can be necessary for certain sign patterns but are never sufficient by themselves. In the next sections we will derive theorems related to these two categories of constraints.

\begin{remark}
The case of $\boldsymbol{p_\alpha}=\boldsymbol{0}$ is the only exception to the condition \eqref{fpcond4}. The attractor point of the negative orthant is then the origin. To find out its stability a more general analysis is required. 
\end{remark}
\begin{exmp} \label{exmp12}
We continue with Example \ref{exmp1}. The attractor points of the parts are
\begin{equation}
\boldsymbol{W}_{\emptyset}=\begin{pmatrix}0\\0\end{pmatrix},\,
\boldsymbol{W}_{\left\{1\right\}}=\begin{pmatrix}1\\2\end{pmatrix},\,
\boldsymbol{W}_{\left\{2\right\}}=\begin{pmatrix}4\\3\end{pmatrix},\,
\boldsymbol{W}_{\left\{1,2\right\}}=\begin{pmatrix}5\\5\end{pmatrix}\,.
\end{equation}
Only $\boldsymbol{W}_{\left\{1,2\right\}}$ satisfies the stable set condition
\begin{equation}
S_{\left\{1,2\right\}}\boldsymbol{W}_{\left\{1\right\}}=
\begin{pmatrix}
1 & 0\\
0 & 1 
\end{pmatrix}
\begin{pmatrix}
5 \\
5
\end{pmatrix}
>\boldsymbol{0}
\,.
\end{equation}
The network has one stable fixed point in the orthant ${\cal P}_{\left\{1,2\right\}}$ (Figure \ref{fig:VPex1}). 
\begin{figure}[ht]
\centering
  \captionsetup{width=.8\textwidth}
\includegraphics[width=0.6\textwidth]{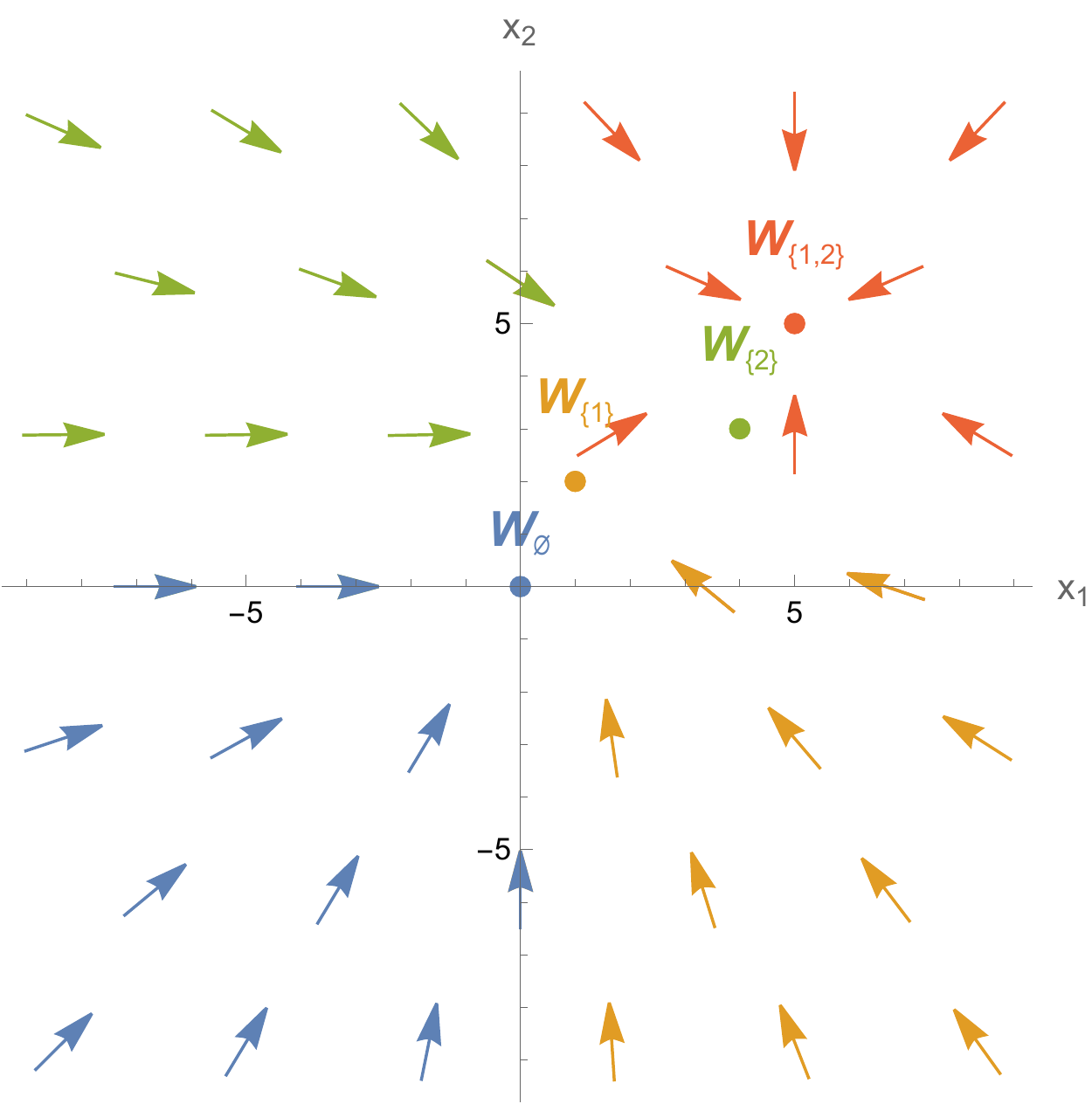}
\caption{Vector field plot of Example \ref{exmp12}. The network has one stable fixed point: $\boldsymbol{W}_{\left\{1,2\right\}}$ . The other three orthants are not stable sets since they don't include their attractor points.}
    \label{fig:VPex1}
\end{figure}
\end{exmp}
\begin{exmp} \label{exmp22}
We continue with Example \ref{exmp2}. The attractor points of the parts are
\begin{equation}
\boldsymbol{W}_{\left\{3\right\}}=\begin{pmatrix}-1\\-1\\1\end{pmatrix},\,\boldsymbol{W}_{\left\{1,3\right\}}=\begin{pmatrix}1\\-1\\1\end{pmatrix},\,\boldsymbol{W}_{\left\{2,3\right\}}=\begin{pmatrix}-1\\1\\1\end{pmatrix},\,\boldsymbol{W}_{\left\{1,2,3\right\}}=\begin{pmatrix}1\\1\\1\end{pmatrix}\,.
\end{equation}
All of the attractor points are stable fixed points since we have that
\begin{align}
S_{\left\{3\right\}}\boldsymbol{W}_{\left\{3\right\}}=
\begin{pmatrix}
-1 & 0 & 0\\
0 & -1 & 0 \\
0 & 0 & 1
\end{pmatrix}
\begin{pmatrix}
-1 \\
-1 \\
1
\end{pmatrix}
&=\begin{pmatrix}
1 \\
1 \\
1 
\end{pmatrix}
>\boldsymbol{0}\,, \nonumber\\
S_{\left\{1,3\right\}}\boldsymbol{W}_{\left\{1,3\right\}}=
\begin{pmatrix}
1 & 0 & 0\\
0 & -1 & 0 \\
0 & 0 & 1
\end{pmatrix}
\begin{pmatrix}
1 \\
-1 \\
1
\end{pmatrix}
&=\begin{pmatrix}
1 \\
1 \\
1 
\end{pmatrix}
>\boldsymbol{0}\,, \nonumber\\
S_{\left\{2,3\right\}}\boldsymbol{W}_{\left\{2,3\right\}}=
\begin{pmatrix}
-1 & 0 & 0\\
0 & 1 & 0 \\
0 & 0 & 1
\end{pmatrix}
\begin{pmatrix}
-1 \\
1 \\
1
\end{pmatrix}
&=\begin{pmatrix}
1 \\
1 \\
1 
\end{pmatrix}
>\boldsymbol{0}\,, \nonumber\\
S_{\left\{1,2,3\right\}}\boldsymbol{W}_{\left\{1,2,3\right\}}=
\begin{pmatrix}
1 & 0 & 0\\
0 & 1 & 0 \\
0 & 0 & 1
\end{pmatrix}
\begin{pmatrix}
1 \\
1 \\
1
\end{pmatrix}
&=\begin{pmatrix}
1 \\
1 \\
1 
\end{pmatrix}
>\boldsymbol{0}\,.
\end{align}
The network has $4$ stable fixed points respectively in the orthants ${\cal P}_{\left\{3\right\}},{\cal P}_{\left\{1,3\right\}},{\cal P}_{\left\{2,3\right\}}$ and ${\cal P}_{\left\{1,2,3\right\}}$ (Figure \ref{fig:VPex2}). 
\begin{figure}[ht]
\centering
  \captionsetup{width=.8\textwidth}
\includegraphics[width=0.6\textwidth]{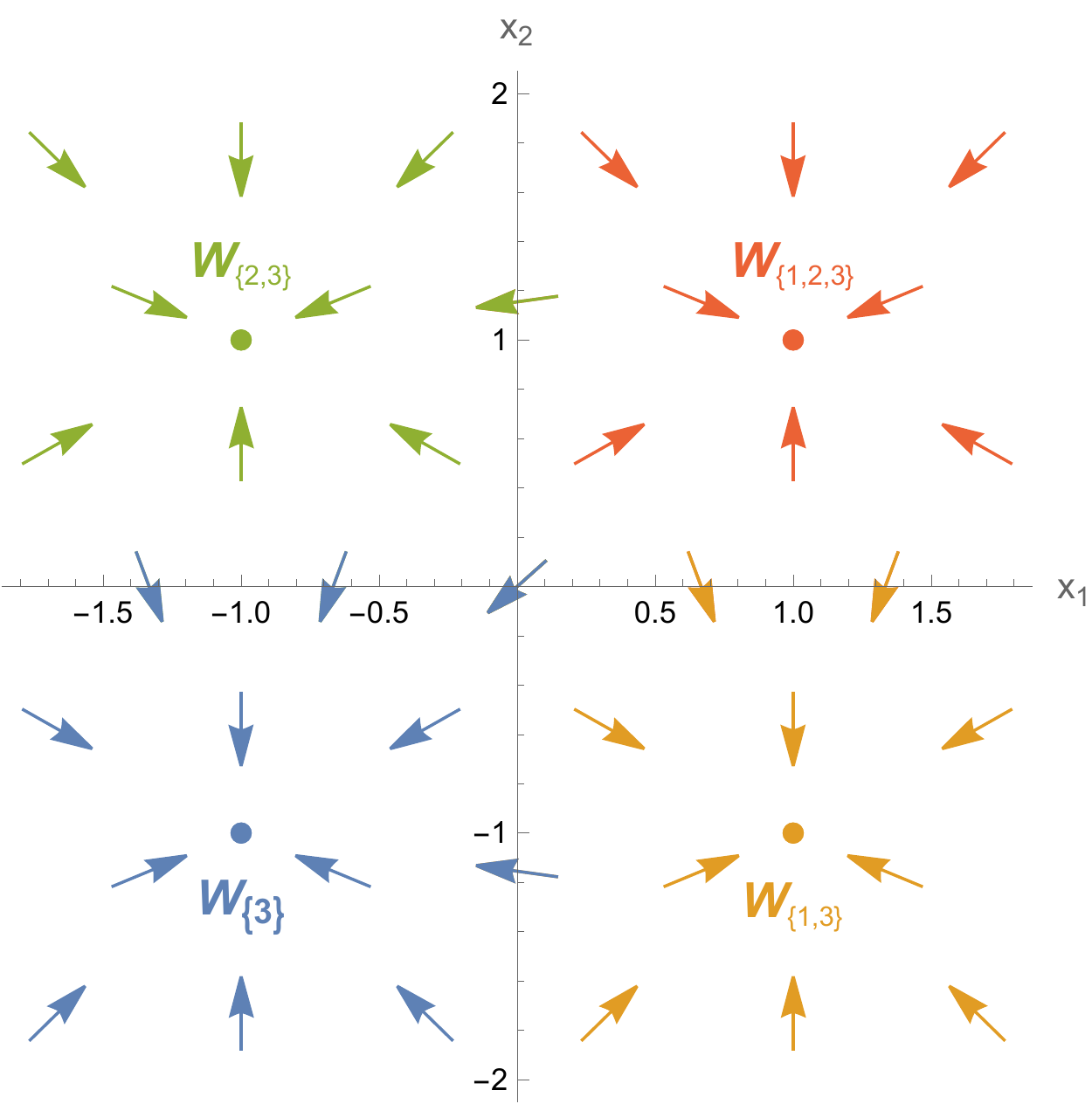}
\caption{Vector field plot of Example \ref{exmp22} projected on the hyperplane $x_3=1$. The network has a stable fixed points in each of the quandrants of the hyperplane.}
    \label{fig:VPex2}
\end{figure}
\end{exmp}

\section{Sign Patterns in Network Connectivity} \label{sec:sp}
\subsection{Signatures} 

The conditions (\ref{fpcond3},~\ref{fpcond4}) for the existence of a stable set put constraints on the possible signs the elements of the weight matrix may have.\footnote{There are multiple representations of \textit{signatures} of a set of elements. We already defined the signature vector of a part in \eqref{signvec}. This vector assigns $+1$ to every (strictly) positive unit and $-1$ to every negative unit. Signatures of an ordered set of elements can be analogously defined by a signature vector, however typically one writes signatures symbolically as an array of $+$'s,$-$'s and $0$'s. For example the signature of $(1,-3,-8,-3,-24,24)$ is $(+,-,-,-,-,+)$, but the signature vector is $(1,-1,-1,-1,-1,1)$. Since both are equivalent ways to represent the signs of elements we will not make a clear distinction and use both representation interchangeably.} This becomes clear when one understands the relationship between linear inequalities and forbidden signatures. To start we consider the simple example of two numbers $a,b \in\mathbb{R}$ that satisfy the inequality
\begin{equation}
a-b>0\,.
\end{equation}
It is clear that $a< 0 $, $b>0$ is not allowed by the order, meaning that the signature $(-,+)$ is forbidden by the linear inequality. Observe that this forbidden, strictly nonzero, signature is the inverse of the signs in the inequality; in that case there would be an overall negative sign and hence a violation of the order. The other signatures $(0,0)$, $(-,0)$ and $(0,+)$ are also disallowed. These are subsignatures of the strictly nonzero signature where in one or more of the positions a nonzero sign is replaced by a zero. The example is easily extended to arbitrary number of elements. The map between linear inequalities containing all the elements of a set and their forbidden strictly nonzero signatures is an isomorphism. 
\subsubsection{Row Signatures}
Now take the $i^{\text{th}}$ row $\boldsymbol{W}^i$ of the weight matrix with elements $\left(w_{i1},w_{i2},\ldots,w_{in}\right)$. Each of these elements can be positive, negative or zero. In total there are $3^n$ possible signatures of which $2^n$ are strictly nonzero. Suppose there is a stable part ${\cal P}_\alpha$. The condition \eqref{fpcond4} imposes the inequalities
\begin{equation}
\sum_{j\,\in\,\alpha}w_{ij}>0 \;\text{ for }\;i\in\alpha  \qquad\text{or}\qquad \sum_{j\,\in\,\alpha}w_{ij}< 0 \;\text{ for }\;i\not\in\alpha\,.
\label{ord}
\end{equation}
The relation between linear inequalities and forbidden signatures immediately implies that in both cases of \eqref{ord} signatures of the row elements get disallowed by the fixed point condition. Those are the ones where for all $j\in\alpha$: $w_{ij} <0$ for $i\in\alpha$ and $w_{ij} >0$  for $i\not\in\alpha$ together with their subsignatures. We can symbolically represent these as
\begin{equation} 
\text{Forbidden signatures:}\quad (\underbrace{-,\ldots,-}_\alpha,\underbrace{*,\ldots,*}_{\alpha^c})+\text{sub.} \quad\text{or}\quad (\underbrace{+,\ldots,+}_{\alpha},\underbrace{*,\ldots,*}_{\alpha^c})+\text{sub.}\,
\label{forbsgn}
\end{equation} 
We use the notation where $*$'s denote arbitrary signs and $\alpha^c=\left[n\right]\setminus\alpha$ is the set complement. A consequence of the forbidden signatures is that there must be at the least a single $w_{ij}>0$ or $w_{ij}< 0$ within the $i$'th row respective of $i\in\alpha$ or $i\in\alpha^c$. 
\subsubsection{Matrix Sign Patterns} Forbidden row elements of $W$ assemble into forbidden \textit{matrix sign patterns}\footnote{Sign patterns are the representations of a signed graph \cite{zaslavsky1982signed}.}. Sign patterns are the sign structures of matrices defined analogously to signature vectors. For practical purposes we'll use the following definition.
\begin{definition}
The sign pattern ${\cal S}$ of a matrix $W$ is a matrix of $-1$, $0$ and $1$'s such that
\begin{equation}
W= W^+\circ{\cal S}
\end{equation}
with ``$\circ$" the Hadamard (elementwise) product and $W^+$ a positive matrix. 
\end{definition}
It follows immediately from the definition that the sign pattern of a matrix $W$ is unique. Suppose there would be two dissimilar sign patterns ${\cal S}_1$, ${\cal S}_2$ for a given $W$. Then there are two positive matrices $W^+_1$, $W^+_2$ such that
\begin{equation}
W^+_1\circ{\cal S}_1=W^+_2\circ{\cal S}_2\quad\Rightarrow\quad {\cal S}_1=\left( W^+_2\oslash W^+_1\right)\circ{\cal S}_2\,,
\end{equation} 
where ``$\oslash$" is the Hadamard (elementwise) division. Since $W^+_2\oslash W^+_1$ is positive, all signs and zeros of ${\cal S}_2$ are preserved and hence ${\cal S}_1={\cal S}_2$ which is a contradiction.

Matrices with the same sign pattern form a \textit{sign pattern class} ${\cal Q}\left({\cal S}\right)$. In the next sections we will identify the set of sign patterns that are \textit{necessary} or \textit{sufficient} for the existence of stable sets. A sign pattern \textit{allows} the stability of a set $\alpha$ if there is a matrix in ${\cal Q}\left({\cal S}\right)$ that satisfies the stability condition \eqref{fpcond4}. A sign pattern \textit{requires} the stability of a set $\alpha$ if all matrices in ${\cal Q}\left({\cal S}\right)$ satisfy the stability condition \eqref{fpcond4}. The set of all allowed sign patterns are the necessary sign patterns. The set of all required sign patterns are the sufficient sign patterns.
\begin{exmp} \label{exmp23}
The matrix in Example \ref{exmp2}
\begin{equation}
W=\begin{pmatrix}
2 & 0 & -1 \\
0 & 2 & -1 \\
0  & 0 & 1
\end{pmatrix}
\end{equation}
has the sign pattern matrix
\begin{equation}
{\cal S}=\begin{pmatrix}
1 & 0 & -1 \\
0 & 1 & -1 \\
0  & 0 & 1
\end{pmatrix}\,.
\end{equation}
Taking the positive matrix
\begin{equation}
W^+=\begin{pmatrix}
2 & \epsilon & 1 \\
\epsilon & 2 & 1 \\
\epsilon  & \epsilon & 1
\end{pmatrix}\,,
\end{equation}
with $\epsilon>0$, one finds that $W=W^+\circ{\cal S}$.
\end{exmp}\subsection{Necessary Sign Patterns}
\subsubsection{Single Stable Set}
\begin{theorem} \label{thm:nsp}
A set $\alpha$ is allowed to be stable iff for each row $i\in [n]$ there is at least one connection $w_{ij}$ where $j\in\alpha$ with the sign $s_\alpha^i$.
\end{theorem}
\begin{proof}
We start by writing the weight matrix as a Hadamard product of a positive matrix $W^+$ with a sign pattern matrix
\begin{equation}
W=W^+\circ {\cal S}\,.
\label{Hprod}
\end{equation}
Making use of the fact that $S_\alpha$ and $P_\alpha$ are diagonal matrices we rewrite, after plugging in \eqref{Hprod}, the stable set condition \eqref{fpcond4} into
\begin{equation}
W^+\circ\left(S_\alpha {\cal S}P_\alpha\right)\boldsymbol{1}> 0\,.
\end{equation}
Furthermore since that
\begin{equation}
\left(A\circ B\right) \boldsymbol{x}=\mathrm{diag}\left(A D_x B^{T}\right)
\end{equation}
with $D_x$ the corresponding diagonal matrix of $\boldsymbol{x}$, the stable set condition can be reformulated as
\begin{equation}
\mathrm{diag}\left(W^+\left[S_\alpha{\cal S}P_\alpha\right]^T\right)> 0\,.
\label{fpcond5}
\end{equation}
The inequality can be satisfied if and only if at least one element in each row of $S_\alpha {\cal S} P_\alpha$ is positive
\begin{equation}
\forall i,\,\exists j : s_\alpha^i {\cal S}^{ij} p_\alpha^j > 0\,.
\end{equation}
This immediately implies that
\begin{equation}
\forall i,\,\exists j\in\alpha : S^{ij}=s_\alpha^i\,.
\label{scond}
\end{equation}
\end{proof}
This result is a variation of what is known in the literature as the sign patterns that allow semipositivity \cite{johnson_smith_tsatsomeros_2020,johnson1993qualitative}. There are two corollaries that follow. 
\begin{corollary} \label{cor:EI}
If $\alpha$ is a $k$-sized stable set then there are at least $k$ excitatory and $n-k$ inhibitory connections.
\end{corollary}
\begin{proof}
If $\alpha$ is stable then by Theorem \ref{thm:nsp} the sign pattern matrix must obey \eqref{scond}. Taking the sum of the positive sign elements over all the rows, we are led to the bound
\begin{equation}
E\geq\sum_i p_\alpha^i=|\alpha|=k
\end{equation}
on the number of excitatory connections $E$. Similarly for the sum of negative sign elements we find the bound
\begin{equation}
I\geq\sum_i p_{\alpha^c}^i=\sum_i\left(1-p_\alpha^i\right)=n-k
\end{equation}
on the number of inhibitory connections $I$.
\end{proof}
The lower bounds on the excitatory and inhibitory connections imposed by a stable set of size $k$ are displayed in Figure \ref{fig:EIbs}. Larger sets require more excitatory connections to be in an active state. The inverse is true for small sets; in order to silence the external units, more inhibitory connections required.
\begin{figure}[ht]
\centering
  \captionsetup{width=.8\textwidth}
\includegraphics[width=0.7\textwidth]{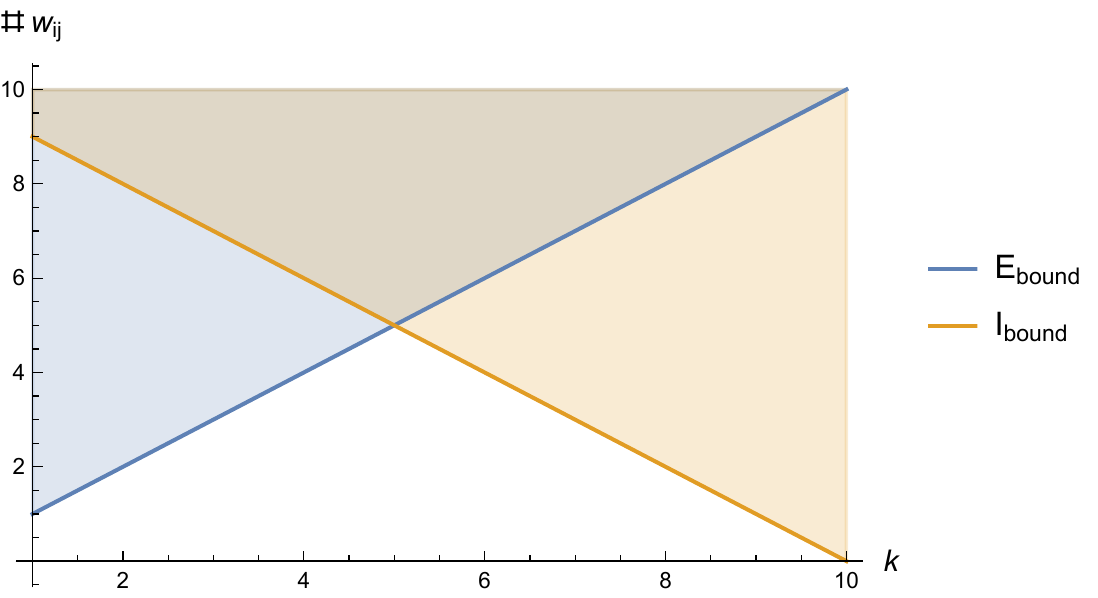}
\caption{Stable sets of size $k$ provide lower bounds on the amount of excitatory ($\text{E}_{\text{bound}}$) and inhibitory ($\text{I}_{\text{bound}}$) connections.}
    \label{fig:EIbs}
\end{figure}
\begin{corollary} \label{cor:fs}
If $\alpha$ is a $k$-sized stable set then the number of allowed signatures for each row of $W$ are $3^n\big(1-\left(\ft{2}{3}\right)^{k}\big)$ .
\end{corollary}
\begin{proof}
Like in \eqref{forbsgn} we can represent the forbidden signatures by their strictly nonzero signatures plus entrywise zero element variations. By \eqref{fpcond5} and the isomorphism between inequality and strictly positive forbidden signatures, the forbidden sign patterns can be written as
\begin{equation}
{\cal S}_{\text{forbidden}}=-\boldsymbol{s_\alpha} \boldsymbol{p}_{\boldsymbol{\alpha}}^T+{\cal S}_{\text{forbidden}}P_{\alpha^c}+\text{sub.}\,.
\label{dissp}
\end{equation}
For each row there are $3^{n-k}$ forbidden signatures coming from the unconstrained elements in the columns of ${\cal S}_{\text{forbidden}}P_{\alpha^c}$. There are $2^k$ subsignatures. We can therefore conclude that a single admitting part of rank $k$ will disallow $2^k\cdot3^{n-k}$ signatures in each row.
\end{proof}
In Figure \ref{fig:alsps} we show the effect of the existence of a stable set of size $k$ on the percentage of allowed row signatures and sign patterns. 
\begin{figure}[ht]
\centering
  \captionsetup{width=.8\textwidth}
\includegraphics[width=1\textwidth]{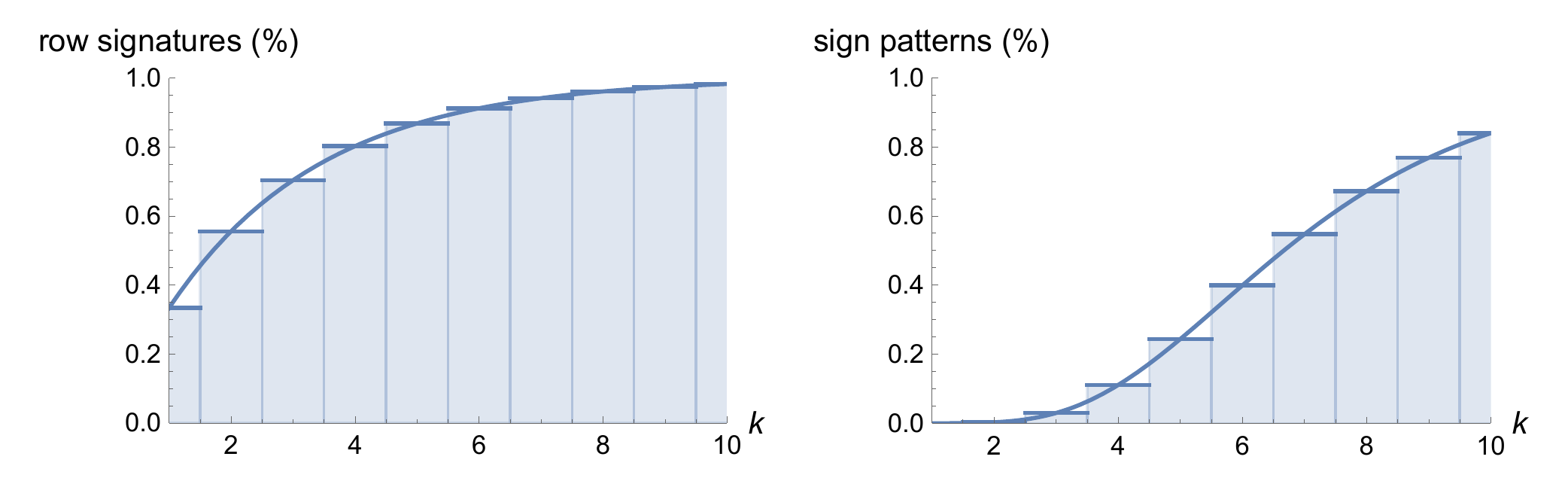}
\caption{Stable sets of size $k$ constrain the allowed row signatures and sign patterns of the network.}
    \label{fig:alsps}
\end{figure}
\begin{remark}
Symbolically one can represent the structure of a sign pattern allowed by the stable set condition \eqref{fpcond4} by writing them in a row/column-permuted form
\vspace{0.75em}

\begin{equation}
\begin{matrix}
Q{\cal S}R=
 \begin{pmatrix}
 \bovermat{$\alpha$}{\quad\boldsymbol{+}&} \hspace{0.7em} \bovermat{$\alpha^c$}{&\boldsymbol{*}\quad} \\[0.7em]

\quad\boldsymbol{-}& \hspace{0.7em} &\boldsymbol{*}\quad \\[0.7em]
  \end{pmatrix}
  \begin{aligned}
  &\left.\begin{matrix}
  \partialphantom    \\
  \end{matrix} \right\} %
  \alpha\\
  &\left.\begin{matrix}
  \partialphantom   \\
  \end{matrix}\right\}%
  \alpha^c\\
 \end{aligned}
 \end{matrix}\,
\end{equation}
where $Q$, $R$ are permutation matrices. The bold $\boldsymbol{+}$/$\boldsymbol{-}$ sign represents the submatrices for which the total sum of their columns are positive/negative. Sign patterns for which these submatrices have rows with only negative/positive elements are hence forbidden. Per row there are $3^{n-k}$ of such combinations. If we include the subsignatures we have $2^{k}\cdot 3^{n-k}$ different sign patterns. The rows of the submatrices need to have at least have one positive/negative element. There are therefore bounds on the number of excitatory and inhibitory connections. Notice that for simplicity we didn't control for the constraint \eqref{nvi}, \eqref{nvi2} in the calculation in Corollary \ref{cor:fs}. For an analysis that includes the additional restrictions see Appendix \ref{app:bsp}.
\end{remark}
\begin{exmp}
Take a 2-dim network. 
\begin{itemize}
\item By Theorem \ref{thm:nsp} sign patterns that respectively allow the sets $\left\{1\right\}$, $\left\{2\right\}$ and $\left\{1,2\right\}$ to be stable are
\begin{align}
\text{allows stability of }\left\{1\right\}&:\quad
\begin{pmatrix}
1 & * \\
-1 & *
\end{pmatrix}\nonumber\\
\text{allows stability of }\left\{2\right\}&:\quad 
\begin{pmatrix}
* & -1 \\
* & 1
\end{pmatrix}\nonumber\\
\text{allows stability of }\left\{1,2\right\}&:\quad
\begin{pmatrix}
1 & * \\
1 & *
\end{pmatrix},\,
\begin{pmatrix}
1 & * \\
* & 1
\end{pmatrix},\,
\begin{pmatrix}
* & 1 \\
1 & *
\end{pmatrix},\,
\begin{pmatrix}
* & 1 \\
* & 1
\end{pmatrix}\, .\label{ss2}
\end{align} 
As a counterexample notice that the weight matrix
\begin{equation}
W=\begin{pmatrix}
-a & 0 \\
c & d
\end{pmatrix}
\end{equation}
with $a>0$ does not satisfy the stability condition for $\left\{1,2\right\}$.

\item There is at least one excitatory and one inhibitory connection in the weight matrix that stabilizes $\left\{1\right\}$ or $\left\{2\right\}$. The network with stable set $\left\{1,2\right\}$ has two excitatory connections. These observations are in line with Corollary \ref{cor:EI}. 

\item At last the distinct necessary sign patterns for $\left\{1,2\right\}$ can be counted from \eqref{ss2}. Each row allows five signatures: $(1,1)$, $(1,0)$, $(1,-1)$, $(0,1)$ and $(-1,1)$. Using the formula in Corollary \ref{cor:fs} we find the same result
\begin{equation}
\# \text{allowed signatures per row}:\quad 3^2\left(1-\left(\frac{2}{3}\right)^2\right)= 5\,.
\end{equation}
\end{itemize}
\end{exmp}
\subsubsection{Family of Disjoint Stable Sets}
We will now head on to the more general case of families of stable sets. We first look at the necessary sign patterns for allowing a family of mutually disjoint stable sets $\left(\alpha_\mathtt{i}\right)_{\mathtt{i}\in \mathtt{I}}$, i.e. $\alpha_{\mathtt{i}}\cap\alpha_\mathtt{i-1}=\emptyset$ for $\mathtt{i}\neq\mathtt{j}$.
\begin{theorem} \label{thm:stsg}
A family of mutually disjoint sets $\left(\alpha_\mathtt{i}\right)_{\mathtt{i}\in \mathtt{I}}$ with sizes $|\alpha_\mathtt{i}|=k_\mathtt{i}$ is allowed to be stable iff for each row $i\in [n]$ and $\mathtt{i}\in\mathtt{I}$ there is at least one connection $w_{ij}$ where $j\in\alpha_{\mathtt{i}}$ with the sign $s_{\alpha_{\mathtt{i}}}^i$.
\end{theorem}
\begin{proof}
Because the sets are mutually disjoint their stability conditions in \eqref{fpcond4} constrain the signs of different columns and hence are independent. The theorem therefore follows by applying Theorem \ref{thm:nsp} for each of the sets.
\end{proof}
\begin{corollary} \label{cor:EId}
If a mutually disjoint family of sets $\left(\alpha_\mathtt{i}\right)_{\mathtt{i}\in \mathtt{I}}$ with sizes $|\alpha_\mathtt{i}|=k_\mathtt{i}$ is stable then there are at least $\sum_{\mathtt{i}\in\mathtt{I}}k_\mathtt{i}$ excitatory and  $n|\mathtt{I}|-\sum_{\mathtt{i}\in\mathtt{I}}k_\mathtt{i}$ inhibitory connections.
\end{corollary}
\begin{proof}
From Theorem \ref{thm:stsg} we know that a stable set requires at least one weight with sign $s_\alpha^i$ in row $i$. The number of positive, negative weights in the row are thus respectively
\begin{equation}
\sum_{\mathtt{i}\in\mathtt{I}} p_{\alpha_{\mathtt{i}}}^i\,,\qquad \sum_{\mathtt{i}\in\mathtt{I}} p_{\alpha^c_{\mathtt{I}}}^i=|\mathtt{I}|-\sum_{\mathtt{i}\in\mathtt{I}} p_{\alpha_{\mathtt{i}}}^i\,.
\end{equation} 
Taking the sum over the rows we can conclude that the disjoint family of sets has at least
\begin{align}
\sum_{i\in \left[n\right]}\sum_{\mathtt{i}\in\mathtt{I}} p_{\alpha_{\mathtt{I}}}^i&=\sum_{\mathtt{i}\in\mathtt{I}}k_{\mathtt{i}}
\end{align}
excitatory connections and at least
\begin{equation}
\sum_{i\in \left[n\right]}\left(|\mathtt{I}|-\sum_{\mathtt{i}} p_{\alpha_{\mathtt{I}}}^i\right)=n|\mathtt{I}|-\sum_{\mathtt{i}\in\mathtt{I}}k_{\mathtt{i}}
\end{equation}
inhibitory connections.
\end{proof}
The lower bounds on the excitatory and inhibitory connections imposed by a disjoint stable sets are displayed in Figure \ref{fig:EIbd}. For fixed amount of stable sets $|\mathtt{I}|$ the bounds linearly depend on the average size of the stable sets $\langle k\rangle $. The absolute number of stable sets determines the slope and height of the bounds. This shows that networks with large reservoirs of stable sets, i.e. large storage capacity, require more fine-tuned wirings. The ratio of the bounds on excitatory and inhibitory connections in terms of the stable state density $\langle k\rangle/n$ of the network are presented in Figure \ref{fig:EIr}. High capacity networks require a lot of internal inhibition to sustain their separate stable states.
\begin{figure}[ht]
\centering
  \captionsetup{width=.8\textwidth}
\includegraphics[width=1\textwidth]{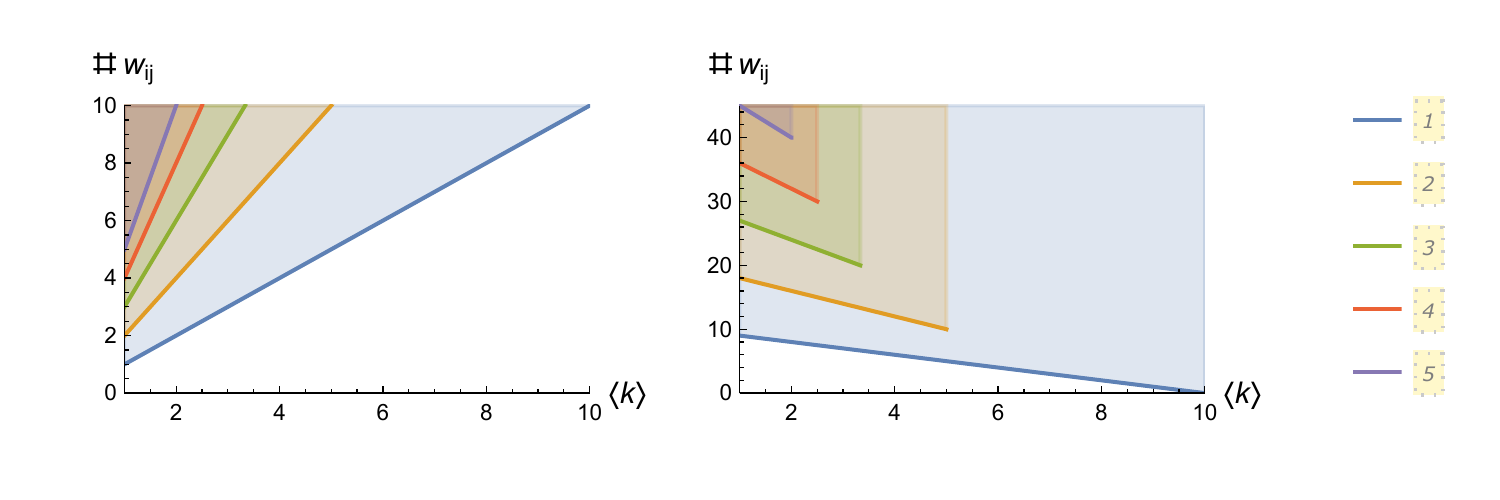}
\caption{Families of disjoint stable sets of average size $\langle k\rangle =\frac1{|\mathtt{I}|}\sum_{\mathtt{i}\in\mathtt{I}}k_{\mathtt{i}}$ provide lower bounds on the amount of excitatory (left figure) and inhibitory (right figure) connections. The bounds are shown for $|\mathtt{I}|=1,\ldots,5$.}
    \label{fig:EIbd}
\end{figure}
\begin{figure}[ht]
\centering
  \captionsetup{width=.8\textwidth}
\includegraphics[width=0.6\textwidth]{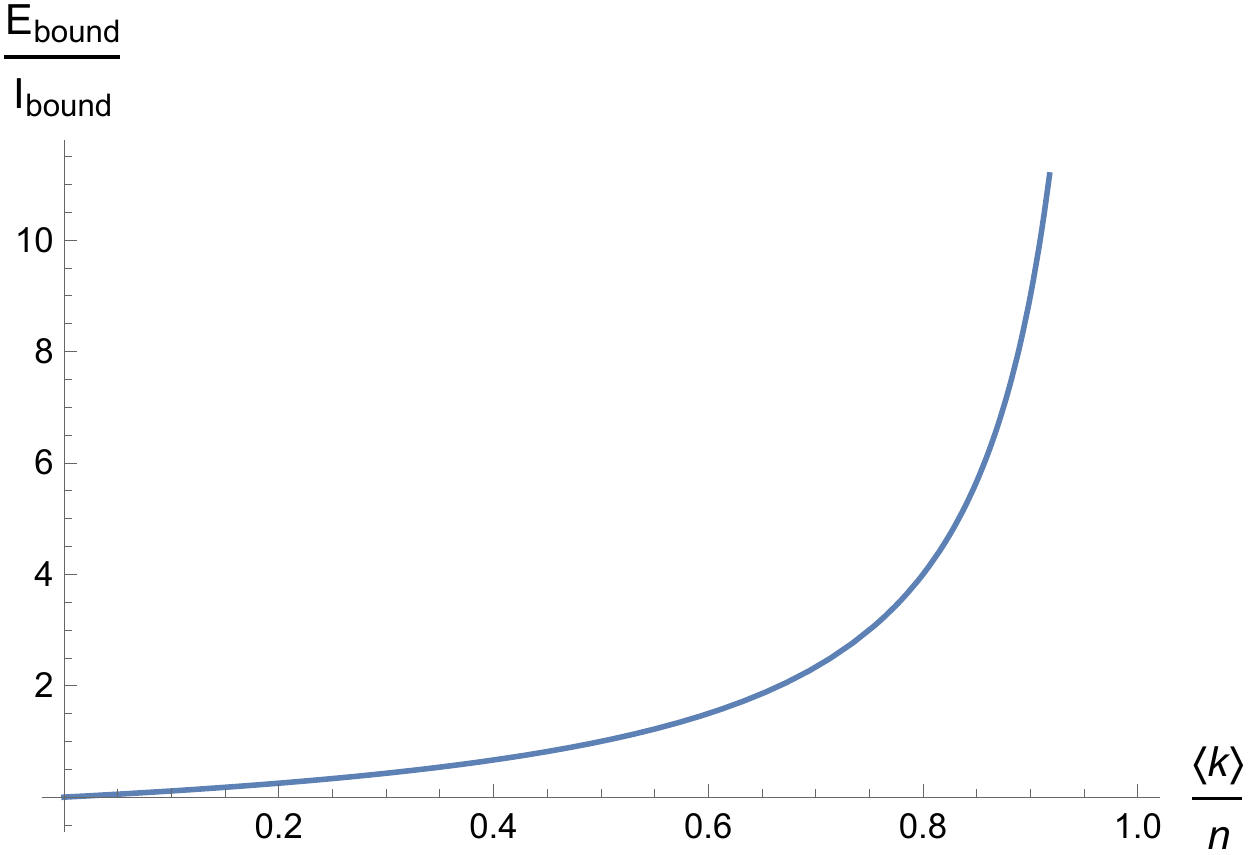}
\caption{Networks with small $\left(\langle k\rangle\ll n\right))$/large $\left(\langle k\rangle\sim n\right)$ disjoint stable sets have larger/smaller inhibitory than excitatory bounds on their connections. Hence generically these networks are likely to contain more inhibitory/excitatory connections.}
    \label{fig:EIr}
\end{figure}
\begin{corollary} \label{cor:fsd}
If a mutually disjoint family of sets $\left(\alpha_\mathtt{i}\right)_{\mathtt{i}\in \mathtt{I}}$ with sizes $|\alpha_\mathtt{i}|=k_\mathtt{i}$ is stable then the number of allowed signatures for each row of $W$ are $3^n\prod_{\mathtt{i}\in \mathtt{I}}\big(1-\left(\ft{2}{3}\right)^{k_\mathtt{i}}\big)$.
\end{corollary}
\begin{proof}
By Corollary \ref{cor:fs} each constraint reduces the allowed set of signatures to
\begin{equation}
3^n-2^{k_\mathtt{i}}\cdot 3^{n-k_\mathtt{i}}=3^{n-k_\mathtt{i}}\left(3^{k_\mathtt{i}}-2^{k_\mathtt{i}}\right)\,.
\label{asign}
\end{equation} 
The constraint disallows $2^{k_\mathtt{i}}$ combinations of the $3^{k_\mathtt{i}}$ possible signatures in the $\alpha_\mathtt{i}$-columns. Other sets will further reduce the factor $3^{n-k_\mathtt{i}}$ in \eqref{asign} in an equivalent way. They will not change the other factor $\left(3^{k_\mathtt{i}}-2^{k_\mathtt{i}}\right)$ since those are the sign combinations of their complement set. The total allowed number of signatures for each row then becomes
\begin{equation}
3^{n-\sum_{\mathtt{i}\in\mathtt{I}}k_\mathtt{i}}\prod_{\mathtt{i}\in \mathtt{I}}\left(3^{k_\mathtt{i}}-2^{k_\mathtt{i}}\right)= 3^n\prod_{\mathtt{i}\in \mathtt{I}}\big(1-\left(\ft{2}{3}\right)^{k_\mathtt{i}}\big)\,.
\end{equation}
\end{proof}
The exponential suppression of the number of allowed row signatures by families of disjoint stable states is displayed in Figure \ref{fig:alspd}. This portrays the power multistability can have in limiting the configuration space of the connections of a neural network.
\begin{figure}[htp]
\centering
\subfloat[The sizes of the individual sets are taken to be equal $k_{\mathtt{i}}=k_{\mathtt{j}}$ for all $\mathtt{i},\mathtt{j}\in\mathtt{I}$.]{%
  \includegraphics[clip,width=0.45\columnwidth]{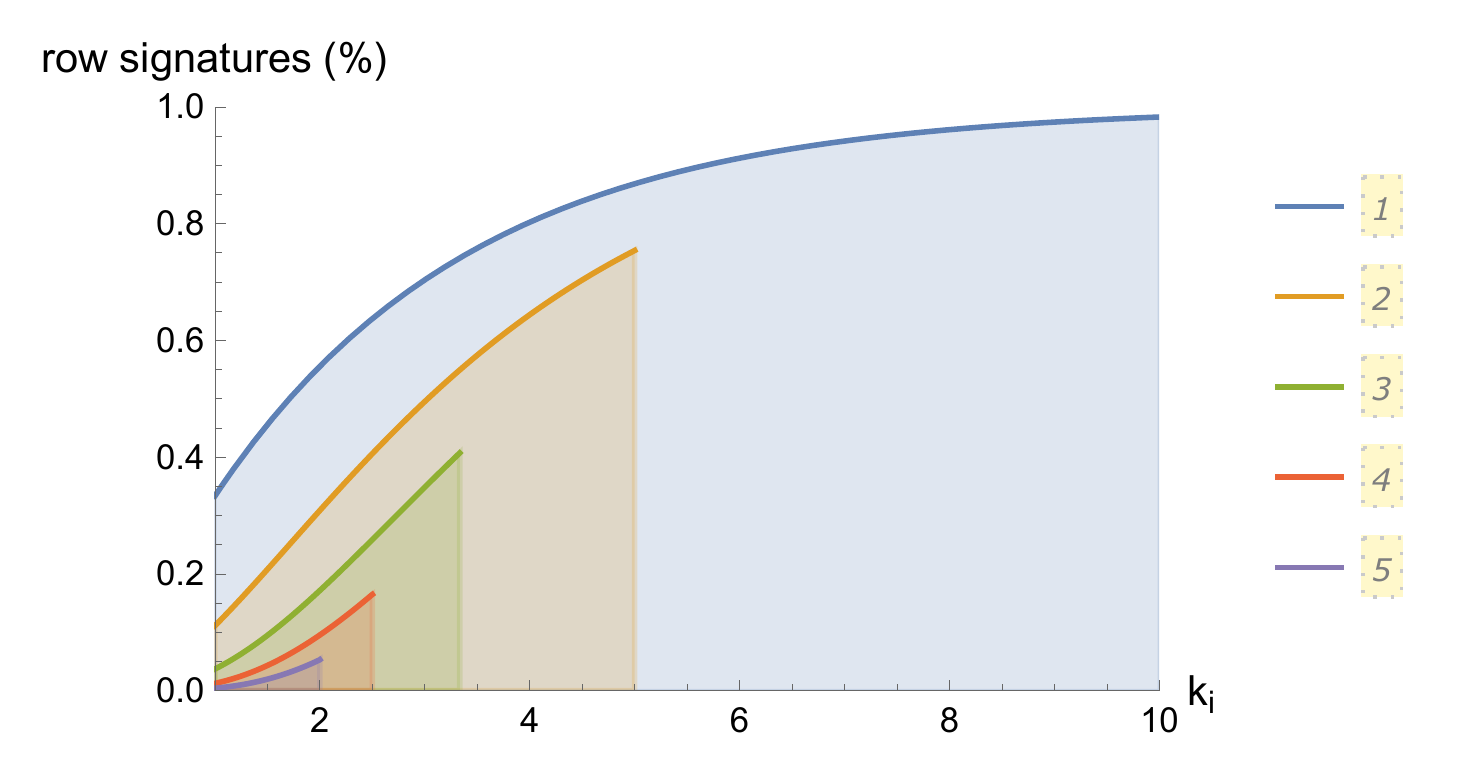}%
}\hspace{0.2cm}
\subfloat[The sizes of the individual sets are fixed to be $k_{\mathtt{i}}=n/|\mathtt{I}|$.]{%
  \includegraphics[clip,width=0.45\columnwidth]{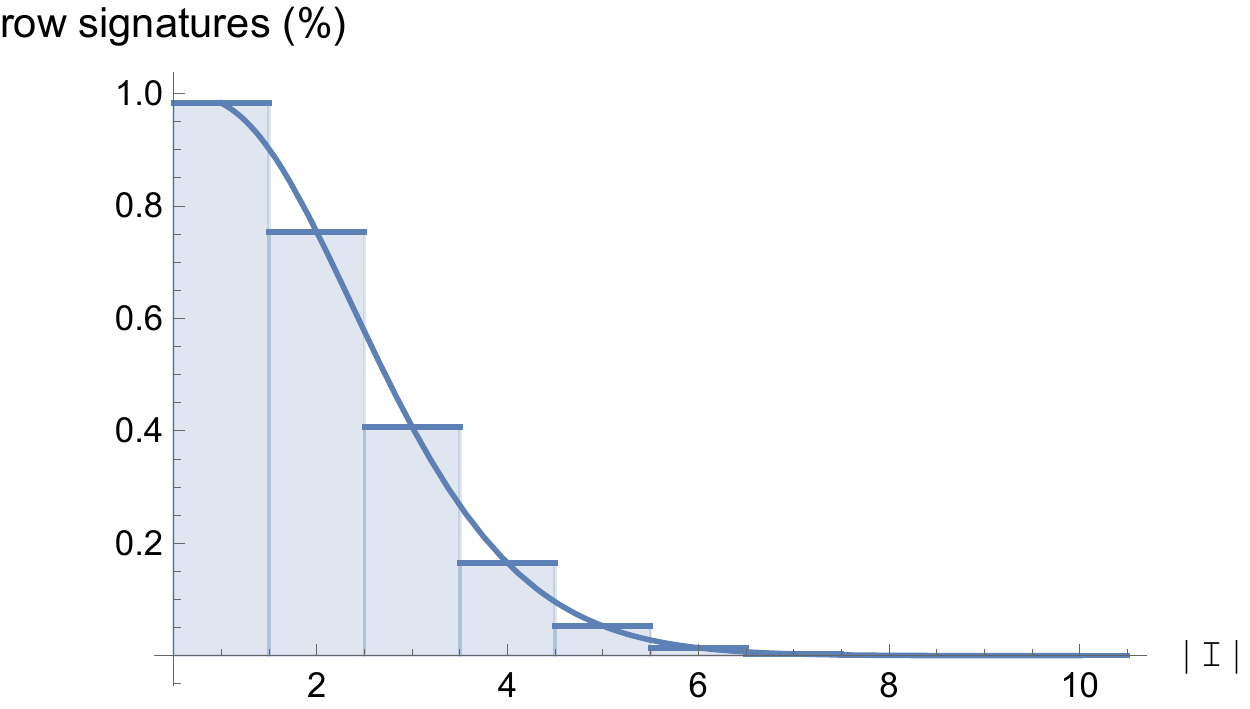}%
}
\caption{The amount of allowed row signatures are suppressed by the amount of disjoint stable sets $|\mathtt{I}|$. }
\label{fig:alspd}
\end{figure}
\begin{exmp}
Suppose a $2$-dim network has $\left\{1\right\}$ and $\left\{2\right\}$ as stable sets.
\begin{itemize}
\item By Theorem \ref{thm:stsg} there is only one allowed sign pattern
\begin{equation} \label{sp12}
\text{allows stability of }\left\{1\right\}\, \& \,\left\{2\right\}:
\begin{pmatrix}
1 & -1 \\
-1 & 1
\end{pmatrix}\,.
\end{equation}
\item There are two excitatory $(E)$ and two inhibitory $(I)$ connections in \eqref{sp12}. This agrees with (and saturates the bounds in) Corollary \ref{cor:EId}
\begin{align}
E&\geq\sum_{\mathtt{i}\in\mathtt{I}}k_\mathtt{i}=1+1=2\nonumber\\
I&\geq n|\mathtt{I}|-\sum_{\mathtt{i}\in\mathtt{I}}k_\mathtt{i}=2\cdot 2-\left(1+1\right)=2\,.
\end{align}
\item By Corollary \ref{cor:fsd} there is exactly
\begin{equation}
3^n\prod_{\mathtt{i}\in \mathtt{I}}\big(1-\left(\ft{2}{3}\right)^{k_\mathtt{i}}\big)=3^2\left(1-\frac23\right)^2=1
\end{equation}
allowed signature for each row and hence one allowed sign pattern. This is what we found in \eqref{sp12}.
\end{itemize}
\end{exmp}
\subsubsection{Family of Nested Stable Sets}
We now give the necessary sign patterns for allowing a family of nested stable sets $\left(\alpha_\mathtt{i}\right)_{\mathtt{i}\in \mathtt{I}}$, i.e. $\alpha_\mathtt{i}\subset\alpha_{\mathtt{j}}$ for $\mathtt{i}<\mathtt{j}$.
\begin{theorem} \label{thm:ns}
A family of nested sets $\left(\alpha_\mathtt{i}\right)_{\mathtt{i}\in \mathtt{I}}$ with sizes $|\alpha_\mathtt{i}|=k_\mathtt{i}$ is allowed to be stable iff for each row $i\in [n]$ there is at least one connection $w_{ij}$ where $j\in\alpha_{\mathtt{1}}$ with sign $s_{\alpha_\mathtt{1}}^i$ and if $i\in\alpha_{\mathtt{i}}\setminus\alpha_{\mathtt{i-1}}$ for some $\mathtt{i}\in\mathtt{I}$ at least one positive connection $w_{ij}$ where $j\in\alpha_{\mathtt{i}}$.
\end{theorem}
\begin{proof}
From Theorem \ref{thm:nsp} we know that for every $\mathtt{i}\in\mathtt{I}$, at all rows $i\in n$ there must be at least one element with sign $s^i_{\alpha_{\mathtt{i}}}$ in a column $j\in\alpha_{\mathtt{i}}$. Because of the nested nature of the sets, the columns that contain an element with a specific sign are included in the larger sets. Hence at the rows where signatures overlap 
\begin{equation}
s_{\alpha_\mathtt{j}}^i=s_{\alpha_{\mathtt{i}}}^i\Leftrightarrow i\in\left(\alpha_{\mathtt{j}}\setminus\alpha_{\mathtt{i}}\right)^c \quad\text{for }\mathtt{j}>{\mathtt{i}}
\end{equation}
the condition \eqref{scond} for $\alpha_{\mathtt{j}}$ is automatically satisfied by the subset $\alpha_{\mathtt{i}}$. This implies that starting from the constraint by the first set $\alpha_{\mathtt{1}}$ all further sign restrictions are on the rows $i\in\alpha_{\mathtt{i}}\setminus\alpha_{\mathtt{i-1}}$ where $\alpha_{\mathtt{i}}$ requires a positive element in a column $j\in\alpha_{\mathtt{i}}$ that is not required by $\alpha_{\mathtt{i-1}}$ and its subsets. 
\end{proof}
\begin{corollary} \label{cor:EIn}
If a nested family $\left(\alpha_\mathtt{i}\right)_{\mathtt{i}\in \mathtt{I}}$ with sizes $|\alpha_\mathtt{i}|=k_\mathtt{i}$ is stable then the number of excitatory connections are at least $k_{\mathtt{N}}$ while there are at least $n-k_\mathtt{1}$ inhibitory connections.
\end{corollary}
\begin{proof}
By Theorem \ref{thm:ns} on row $i\in n$ there is an excitatory connection if $i\in\alpha_\mathtt{N}$, where $\mathtt{N}=\max(\mathtt{I})$. There will be an inhibitory connection if $i\in\alpha_{\mathtt{1}}^c$. We thus find at least
\begin{equation}
p^i_{\alpha_{\mathtt{N}}}\,, \quad 1- p^i_{\alpha_\mathtt{1}}
\end{equation}
excitatory and inhibitory connections on row $i$. The total bounds are
\begin{equation}
E\geq\sum_i p^i_{\alpha_{\mathtt{N}}}=k_{\mathtt{N}}\,,\quad I\geq\sum_{i}\left(1- p^i_{\alpha_\mathtt{1}}\right)=n-k_{\mathtt{1}}\,.
\end{equation}
\end{proof}
For nested stable sets in an $n$-dim network the lower bounds on the excitatory and inhibitory connections are completely determined by the bounds of the minimal and maximal stable set: $\alpha_{\mathtt{1}}$ and $\alpha_{\mathtt{N}}$. The minimal/maximal set puts lower bounds on the inhibitory/excitatory connections. In Figure \ref{fig:EIvp} a vector plot is drawn in terms of the sizes $k_{\mathtt{1}}$ and $k_{\mathtt{N}}$ of these sets. 
\begin{figure}[ht]
\centering
  \captionsetup{width=.8\textwidth}
\includegraphics[width=0.6\textwidth]{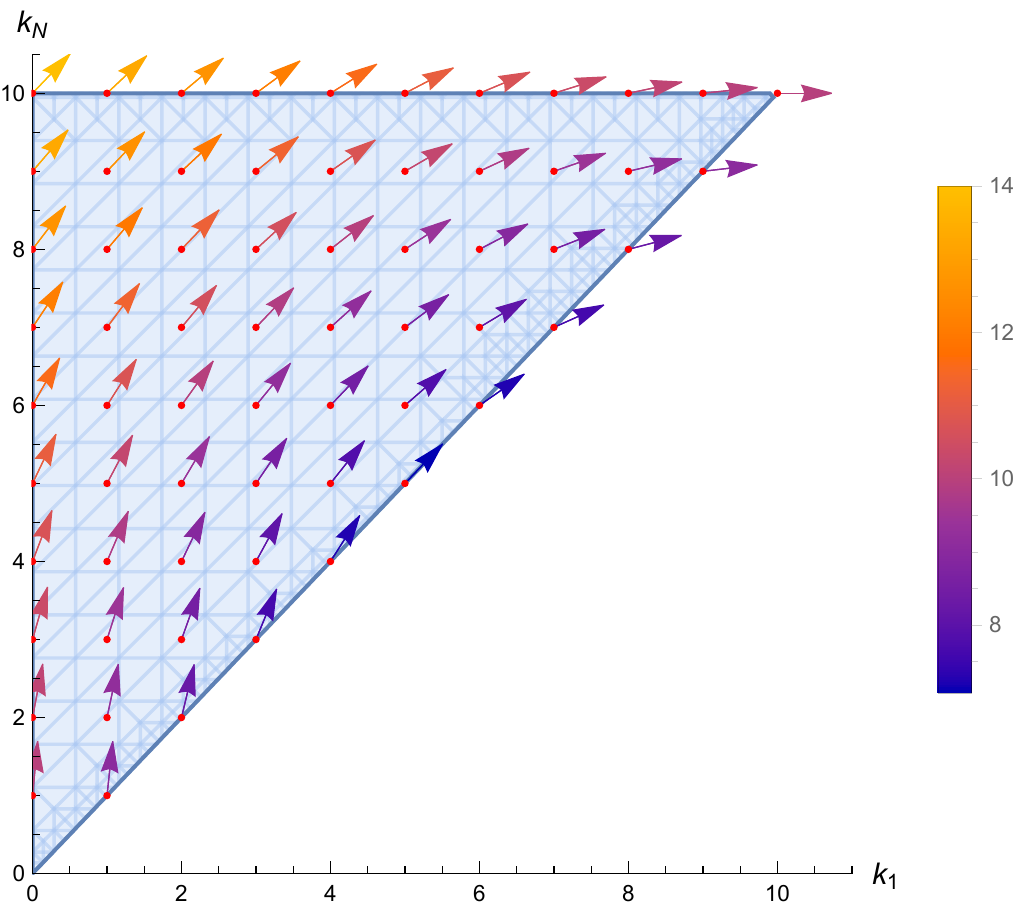}
\caption{A vector plot of $\left(E_{\text{bound}},I_{\text{bound}}\right)$ on the parameter space $\left(k_{\mathtt{1}},k_\mathtt{N}\right)$ for nested stable sets. For larger $k_{\mathtt{1}}$ and $k_{\mathtt{N}}$ bounds become increasingly excitatory. The bounds do not depend on the amount of nested sets $|\mathtt{I}|$. The colors represent the absolute value of the vectors.}
    \label{fig:EIvp}
\end{figure}
\begin{corollary} \label{cor:fsn}
If a nested family $\left(\alpha_\mathtt{i}\right)_{\mathtt{i}\in \mathtt{I}}$ with sizes $|\alpha_\mathtt{i}|=k_\mathtt{i}$ is stable then there are $3^{n}\Big(1-\left(\ft{2}{3}\right)^{k_\mathtt{1}}-\sum_{\mathtt{i}>\mathtt{1}}\left(1-2^{-k_\mathtt{1}}\right)\left(\ft{2}{3}\right)^{k_\mathtt{i}}p_{\alpha_\mathtt{i}\setminus\alpha_{\mathtt{{i-1}}}}^i\Big)$ allowed signatures for the $i^{\text{th}}$ row.
\end{corollary}
\begin{proof}
By Corollary \ref{cor:fs} and Theorem \ref{thm:ns} the set $\alpha_{\mathtt{1}}$ disallows 
\begin{equation}
2^{k_\mathtt{1}}\cdot 3^{n-k_{\mathtt{1}}}
\end{equation}
signatures at each row $i\in \left[n\right]$. The supersets $\alpha_{\mathtt{i}}$, with $\mathtt{i}>\mathtt{1}$, further forbid signatures with no positive element in the columns $j\in\alpha_\mathtt{i}$ on the rows $i\in\alpha_{\mathtt{i}}\setminus\alpha_{\mathtt{i-1}}$. In total those are
\begin{equation}
2^{k_\mathtt{i}}\cdot 3^{n-k_\mathtt{i}}p_{\alpha_{\mathtt{i}}\setminus\alpha_{\mathtt{i-1}}}^i\,.
\end{equation}
The signatures
\begin{equation}
(\underbrace{0,\ldots,0}_{\alpha_{\mathtt{{1}}}},\underbrace{*,\ldots,*}_{\alpha_{\mathtt{i}}\setminus\alpha_{\mathtt{1}}},\underbrace{*,\ldots,*}_{\alpha_{\mathtt{i}}^c})
\end{equation}
are disallowed by the $\alpha_{\mathtt{{1}}}$ set. The subset of signatures that have no $``+"$ in $\alpha_\mathtt{i}\setminus\alpha_\mathtt{1}$, are forbidden by $\alpha_{\mathtt{{i}}}$ when $i\in\alpha_\mathtt{i}\setminus\alpha_\mathtt{i-1}$. There are hence $2^{k_\mathtt{i}-k_\mathtt{1}}\cdot 3^{n-k_\mathtt{i}}$ overlapping forbidden signatures we doubly accounted for. Correcting those, we arrive at the total the number of allowed signatures for row $i$:
\begin{equation}
\underbrace{3^n}_{\text{total}}-\underbrace{2^{k_\mathtt{1}}\cdot 3^{n-k_{\mathtt{1}}}}_{\text{no }``-/+"\text{ by }\alpha_{\mathtt{1}}} -\underbrace{\sum_{\mathtt{i}>\mathtt{1}}2^{k_\mathtt{i}}\cdot 3^{n-k_\mathtt{i}}p_{\alpha_{\mathtt{i}}\setminus\alpha_{\mathtt{i-1}}}^i}_{\text{no } ``-" \text{ in }i\,\in\,\alpha_{\mathtt{i}}\setminus\alpha_{\mathtt{i-1}} \text{ by }\alpha_{\mathtt{i}}}+\underbrace{\sum_{\mathtt{i}>\mathtt{1}}2^{k_\mathtt{i}-k_\mathtt{1}}\cdot 3^{n-k_\mathtt{i}}p_{\alpha_{\mathtt{i}}\setminus\alpha_{\mathtt{i-1}}}^i}_{\text{double counting correction}}\,.
\end{equation}
The total sum can be rewritten into
\begin{equation}
3^{n}\left(1-\left(\ft{2}{3}\right)^{k_\mathtt{1}}-\sum_{\mathtt{i}>\mathtt{1}}\left(1-2^{-k_\mathtt{1}}\right)\left(\ft{2}{3}\right)^{k_\mathtt{i}}p_{\alpha_{\mathtt{i}}\setminus\alpha_{\mathtt{i-1}}}^i\right)\,.
\end{equation}
\end{proof}
One can see in Figure \ref{fig:alspsn} that the allowed sign patterns of a network get strongly suppressed by the storage of multistable nested sets.
\begin{figure}[ht]
\centering
  \captionsetup{width=.8\textwidth}
\includegraphics[width=1\textwidth]{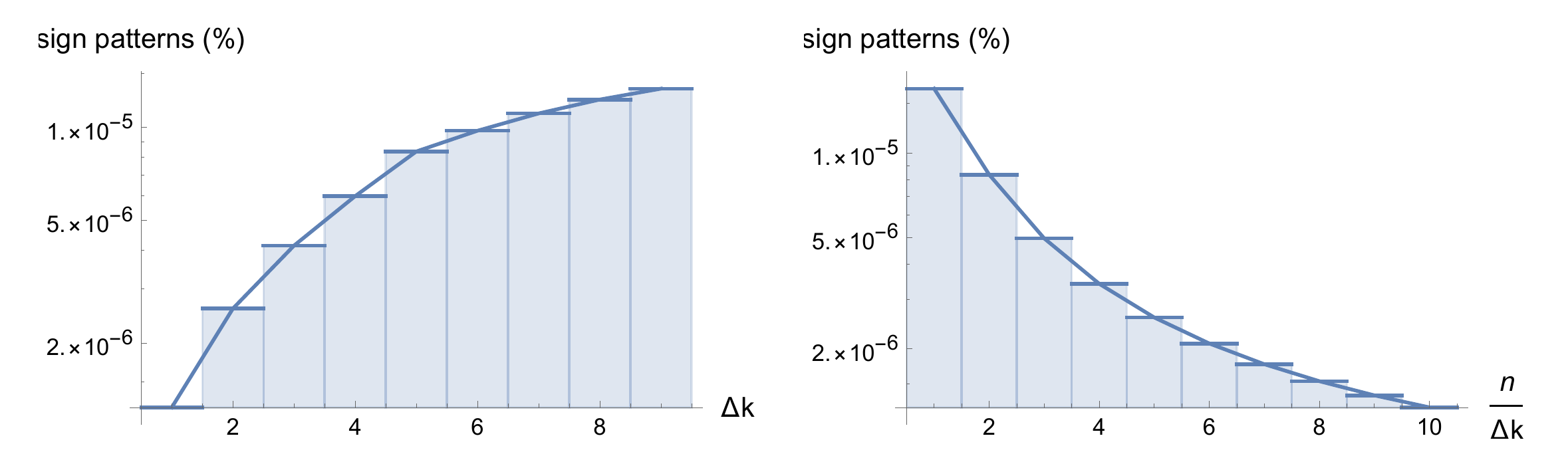}
\caption{The allowed sign patterns of sequences of nested stable sets are plotted. The sequences have fixed $\Delta k = k_{\mathtt{i+1}}-k_{\mathtt{i}}$ and the first set always has size $k_\mathtt{1}=1$. The right figure is obtained from left by inverting the argument into a density $n/\Delta k$.}
    \label{fig:alspsn}
\end{figure}
\begin{exmp}
Suppose a $2$-dim network has $\left\{1\right\}$ and $\left\{1,2\right\}$ as stable sets.
\begin{itemize}
\item By Theorem \ref{thm:ns} the allowed sign patterns are
\begin{equation} \label{spn}
\text{allows stability of }\left\{1\right\}\, \& \,\left\{1,2\right\}:
\begin{pmatrix}
1 & * \\
-1 & 1
\end{pmatrix}\,.
\end{equation}
\item There are at least two excitatory $(E)$ and one inhibitory $(I)$ connections in \eqref{spn}. This agrees with (and saturates the bounds in) Corollary \ref{cor:EIn}
\begin{align}
E&\geq k_{\mathtt{N}}=2\nonumber\\
I& \geq n-k_{\mathtt{1}}=1\,.
\end{align}
\item By Corollary \ref{cor:fsn} there are
\begin{equation}
3^{n}\Big(1-\left(\ft{2}{3}\right)^{k_\mathtt{1}}-\sum_{\mathtt{i}>\mathtt{1}}\left(1-2^{-k_\mathtt{1}}\right)\left(\ft{2}{3}\right)^{k_\mathtt{i}}p_{\alpha_{\mathtt{i}}\setminus\alpha_{\mathtt{{i-1}}}}^i\Big)=
\begin{cases}
3^2\left(1-\frac23\right)=3 &\text{ for }i=1\\
3^2\left(1-\frac23-\frac29\right)=1 &\text{ for }i=2\,.
\end{cases}
\end{equation}
allowed signatures for respectively the first and second row. Hence there are three allowed sign patterns. This is what we found in \eqref{spn}. Notice that we don't count the restrictions imposed by the constraint \eqref{nvi}, \eqref{nvi2} on the allowed sign patterns. More on this can be found in the Appendix \ref{app:bsp}.
\end{itemize}
\end{exmp}
\subsection{Sufficient Sign Patterns}
Sign patterns can require the stability of states independently from the strengths of the connections. The corresponding sets can be called \textit{sign stable}. The upcoming Definitions \ref{def:st}, \ref{def:rp} and Theorem \ref{thm:ssp} make this explicit.
\begin{definition} \label{def:st}
A set $\alpha$ is sign stable whenever it is required to be stable, i.e. stable for all $W^+$.
\end{definition}
\begin{definition} \label{def:rp}
A matrix $A$ is row positive whenever $A\geq 0$ and there is at least one positive element on each row.
\end{definition}
\begin{theorem} \label{thm:ssp}
A set $\alpha$ is sign stable iff $S_\alpha{\cal S}P_\alpha$ is row positive.
\end{theorem}
\begin{proof}
Both directions of the implications easily follow from the result we found in \eqref{fpcond5}:
\begin{equation}
S_\alpha WP_\alpha\boldsymbol 1>0\quad \Leftrightarrow\quad \mathrm{diag}\left(W^+{\cal S}(\alpha)^T\right)> 0\,.
\label{fpcond9}
\end{equation}
where for convenience we defined ${\cal S}(\alpha)\equiv S_\alpha{\cal S}P_\alpha$.

$\Rightarrow.$ Suppose ${\cal S}(\alpha)$ is not row positive. Then there is a row $i\in\left[n\right]$ with at least one element $s_{ij}(\alpha)= -1$ for $j\in\alpha$ or all elements are zero. Taking $(w^+)_{ij}\geq\sum_{k\,\in\, \alpha\setminus\left\{j\right\}} w^+_{ik}$, we then find
\begin{equation}
\left(W^+{\cal S}(\alpha)^T\right)_{ii}= w^+_{ij}s_{ij}(\alpha)+\sum_{k\,\in\, \alpha\setminus\left\{j\right\}} w^+_{ik}s_{ik}(\alpha)\leq 0
\end{equation}
which is in contradiction with $\alpha$ being stable for all $W^+$.

$\Leftarrow.$ For each row $i\in\left[n\right]$, one element of all $s_{ij}(\alpha)\geq 0$ with $j\in\alpha$ is strictly positive. For generic $W^+$ we thus have
\begin{equation}
\left(W^+{\cal S}(\alpha)^T\right)_{ii}=\sum_{j\in\alpha}w^+_{ij}s_{ij}(\alpha)>0\,.
\end{equation}
By \eqref{fpcond9} it follows that $\alpha$ is stable for every $W^+$.
\end{proof}
This result is a variation of what is known in the literature as sign patterns that require semipositivity \cite{johnson_smith_tsatsomeros_2020,johnson1993qualitative,johnson1995sign}. 
\begin{exmp} \label{exmp:ssp}
The network in Example \ref{exmp1} has the following sign pattern
\begin{equation}
{\cal S}=
\begin{pmatrix}
1 & 1  \\
1 & 1  
\end{pmatrix}\,.
\end{equation}
It follows immediately that the matrix
\begin{equation}
S_{\left\{1,2\right\}}{\cal S}P_{\left\{1,2\right\}}=
\begin{pmatrix}
1 & 1  \\
1 & 1  \\
\end{pmatrix}
\end{equation}
is row positive. Hence the set $\left\{1,2\right\}$ is sign stable, i.e. stable for any $W^+$.
\end{exmp}
Both Theorems \ref{thm:nsp} \& \ref{thm:ssp} explicitly show what we alluded to at the end of the subsection \ref{sss:fpc}, i.e. sign patterns provide necessary and sometimes sufficient conditions for steady state dynamics and pattern memorization. Within the context of our nonlinear model we have now showed that stable states can be fully determined by the configuration of the connections, no matter what the strength of these connections are. It would be of interest to further study the conditions under which networks display sign stability. This might be an underexplored mechanism that real neural networks employ, i.e. using cell type configurations, to stabilize the network in ways that are not influenced by synaptic strengthening or weakening. 
\subsection{Minimal Sign Patterns} \label{sec:msp}
Certain stable sets have no proper stable subset. We will call these \textit{minimally stable} analogous to the definition in the context minimal semipositivity \cite{johnson1994semipositivity,werner1994characterizations}. 
\begin{definition} \label{def:ss}
A set $\alpha$ is minimally stable whenever it is stable and has no stable proper subsets.
\end{definition}
Only a subset of all the sign patterns will allow minimal stability and even a smaller subset will require it. The sign patterns that allow minimal semipositivity have not been fully identified \cite{johnson_smith_tsatsomeros_2020,johnson1995sign}. We will not try to answer this question in the context of the current work. When it comes to sign patterns that require minimal semipositivity the result is more straightforward. In our case the sign patterns that require minimal stability completely overlap with those that require stability.
\begin{theorem} \label{thm:msp}
A set $\alpha$ is required to be minimally stable iff it is sign stable.
\end{theorem} 
\begin{proof} 
$\Rightarrow$. Follows by Definition \ref{def:st}.

$\Leftarrow$. Suppose $\alpha$ is sign stable. Take a proper subset $\beta\subsetneq\alpha$. By Theorem \ref{thm:ssp} the sign pattern we have that
\begin{equation}
S_\alpha{\cal S}P_\beta\geq 0\quad\Rightarrow\quad P_{\alpha\setminus\beta}{\cal S}P_\beta\geq 0\,.
\end{equation}
Hence on the rows $i\in\alpha\setminus\beta$ there is no element with negative sign in ${\cal S}P_\beta$. By Theorem \ref{thm:nsp} the subset $\beta$ is not allowed to be stable.
\end{proof}
\begin{exmp}
Returning to Example \ref{exmp1}, we verify that the network is required to be minimally stable for $\left\{1,2\right\}$. For the subsets $\left\{1\right\}$ and $\left\{2\right\}$ we have
\begin{equation}
S_{\left\{1\right\}}{\cal S}P_{\left\{1\right\}}=
\begin{pmatrix}
1 & 0  \\
-1 & 0 \\

\end{pmatrix}\,,\quad
S_{\left\{2\right\}}{\cal S}P_{\left\{2\right\}}=
\begin{pmatrix}
0 & -1  \\
0 & 1  \\
\end{pmatrix}\,. \label{minsp}
\end{equation}
Following from Theorem \ref{thm:nsp} the first/second column in respectively the first/second matrix in \eqref{minsp} should be strictly positive in order for $\left\{1\right\}$/$\left\{2\right\}$ to be allowed to be stable. Since this isn't the case, the network is minimally stable for the sets $\left\{1,2\right\}$, which coincides with Theorem \ref{thm:msp}.
\end{exmp}
\section{Factorization, (De)composability \& Coupling} \label{sec:fdi}
This section is split into three parts. In the first we'll present a factorization theorem for semipositive matrices. In the second we prove two main theorems. The first one is a composition theorem which formulates when two stable sets can give rise to a stable superset. The other is a decomposition theorem which gives the conditions under which a stable set allows for two stable subsets. In the last part we state the general coupling theorem between stable sets of which the composition and decomposition theorem are special cases. 
\subsection{Factorization} \label{subs:fac}
Semipositive matrices are characterized by a factorization theorem \cite{tsatsomeros2016geometric}. We provide a slightly modified version that lends itself to the context of stable sets.
\begin{lemma} \label{lem:fact}
The matrix $A\in M_{m,n}\left(\mathbb{R}\right)$ has a semipositivity vector $\boldsymbol{x}\geq 0$ iff there exist a nonnegative matrix $X\in M_{n}\left(\mathbb{R}\right)$ and positive matrix $Y\in M_{m,n}\left(\mathbb{R}\right)$ such that the inverse $X^{-1}$ with seminonnegativity vector $\boldsymbol{x}$ exists and $A=YX^{-1}$.
\end{lemma}
\begin{proof}
$\Rightarrow$. By definition there exists a positive vector $\boldsymbol{y}$ such that $A\boldsymbol{x}=\boldsymbol{y}$. The matrices $X$, $Y$ are constructed as
\begin{equation}
X=\boldsymbol{x}\boldsymbol{1_n}^T+\epsilon \mathds{1_n}\,,\quad Y=\boldsymbol{y}\boldsymbol{1_n}^T+\epsilon A
\label{XandY}
\end{equation}
with $\epsilon>0$ chosen small enough such that $Y$ is positive. Notice that $AX=Y$. The matrix $X$ is invertible with positive eigenvalues and eigenvectors
\begin{equation}
X \boldsymbol{x}=\big(\sum_i x^i+\epsilon\big)\boldsymbol{x}\,,\quad X \boldsymbol{z}=\epsilon\boldsymbol{z} \quad\text{for }\boldsymbol{z}\bot\boldsymbol{1}\,.
\label{eigX}
\end{equation}
It immediately follows from \eqref{eigX} that $\boldsymbol{x}$ is a seminonnegativity vector of $X^{-1}$.

$\Leftarrow$. Because $\boldsymbol{x}$ is a seminonnegativity vector for $X^{-1}$ we have that $X^{-1}\boldsymbol{x}\geq \boldsymbol{0}$ and since $Y$ is positive and $\boldsymbol{x}\neq\boldsymbol{0}$ this implies that $YX^{-1}\boldsymbol{x}>\boldsymbol{0}$.
\end{proof}
By the semipositivity condition in \eqref{fpcond4} this has the following consequences for a weight matrix with a stable set. 
\begin{theorem}\label{thm:fac} The following conditions are equivalent:
\begin{enumerate}
\item The set $\alpha$ is stable.
\item There is a nonnegative matrix $X$ and positive matrix $Y$ such that the inverse $X^{-1}$ with seminonnegativity vector $\boldsymbol{p_\alpha}$ exists and $W=S_\alpha Y X^{-1}$.
\item For the principal submatrix $W[\alpha]$ and row complement $W[\alpha^c,\alpha]$ there are positive matrices $X_\alpha, X_{\alpha^c},Y_\alpha,$ and negative $Y_{\alpha^c}$ such that the inverses $X_\alpha^{-1}$, $X^{-1}_{\alpha^c}$ with semipositivity vector $\boldsymbol{1}$ exist and $W[\alpha]=Y_\alpha X_\alpha^{-1}$, $W[\alpha^c,\alpha]=Y_{\alpha^c} X_{\alpha^c}^{-1}$.
\end{enumerate}
\end{theorem}
\begin{proof}\leavevmode
\begin{itemize}
\item $(1)\Leftrightarrow(2)$. By the stable set condition \eqref{fpcond4}, $\alpha$ is a stable set iff $S_\alpha W \boldsymbol{p_\alpha}>\boldsymbol{0}$. By Lemma \ref{lem:fact} and since $S_\alpha^2=\mathds{1}$ both directions of the implication are proved.
\item $(1)\Leftrightarrow(3)$. The stable set condition decomposed in terms of submatrices is $W[\alpha]\,\boldsymbol{1}>\boldsymbol{0}$ and $W[\alpha^c,\alpha]\,\boldsymbol{1}<\boldsymbol{0}$. Both directions of the proof follow immediately from Lemma \ref{lem:fact}. Notice that by construction in \eqref{XandY} the matrices $X_\alpha, X_{\alpha^c},Y_\alpha$ are positive and $Y_{\alpha^c}$ is negative. 
\end{itemize}
\end{proof}
\begin{exmp}
Since the set $\left\{1,2,3\right\}$ is stable in the network of Example \ref{exmp2} we can construct the positive matrices
\begin{equation}
X=
\frac12\begin{pmatrix}
3 & 2 & 2\\
2 & 3 & 2 \\
2 & 2 & 3
\end{pmatrix}\quad \&\quad
Y=
\frac12\begin{pmatrix}
4 & 2 & 1 \\
2 & 4 & 1 \\
2 & 2 & 3 
\end{pmatrix}\,.
\end{equation}
Notice that the inverse of matrix $X$ is semipositive for $\boldsymbol{p}_{\left\{1,2,3\right\}}$
\begin{equation}
X^{-1}\boldsymbol{p}_{\left\{1,2,3\right\}}=
\frac17
\left(
\begin{array}{ccc}
 10 & -4 & -4 \\
 -4 & 10 & -4 \\
 -4 & -4 & 10 \\
\end{array}
\right)
\begin{pmatrix}
1 \\
1\\
1
\end{pmatrix}
=\frac27
\begin{pmatrix}
1 \\
1\\
1
\end{pmatrix}
>\boldsymbol{0}
\end{equation}
and that
\begin{equation}
YX^{-1}=
\left(
\begin{array}{ccc}
 2 & 0 & -1 \\
 0 & 2 & -1 \\
 0 & 0 & 1 \\
\end{array}
\right)
\end{equation}
is equal to $W$ in \eqref{W2}.
\end{exmp}
\subsection{(De)composability} \label{subs:dc}
The existence of a stable set $\alpha$ is not on its own a sufficient or necessary condition for the existence of stable sub- or supersets. Only in the case where the weight matrix is assumed to have a certain structure (as in Section \ref{sec:msp}) direct implications are possible. In this section we will look at the sufficient and necessary conditions on the weight matrix in order for stable sets to (de)compose into stable (sub-)supersets. These are of interest for the pattern completing capabilities of the network. The following matrix structures will be of relevance.
\begin{definition} \label{def:blockdd}
A matrix $W$ is $\alpha$-block diagonally dominant whenever 
\begin{equation}
\left|\sum_{j\in\alpha} w_{ij}\right|\geq \left|\sum_{j\not\in\alpha} w_{ij}\right| \text{ for } i\in\alpha
\end{equation}
and
\begin{equation}
\left|\sum_{j\not\in\alpha} w_{ij}\right|\geq \left|\sum_{j\in\alpha} w_{ij}\right| \text{ for } i\not\in\alpha
\end{equation}
\end{definition}
\begin{definition} \label{def:blockZ}
A matrix $W$ is an $\alpha$-block $Z$-matrix whenever 
\begin{equation}
\sum_{j\not\in\alpha} w_{ij}\leq 0 \text{ for } i\in\alpha
\end{equation}
and
\begin{equation}
\sum_{j\in\alpha} w_{ij}\leq 0 \text{ for } i\not\in\alpha\,.
\end{equation}
\end{definition}
\begin{remark}
In the context of our Glass model the Definitions \ref{def:blockdd},\ref{def:blockZ} allow simple expressions in terms of attractor points \eqref{fcp}
\begin{align}
\text{Definition } \ref{def:blockdd}:\;& s^i_\alpha\left|W_\alpha^i\right|+s^i_{\alpha^c}\left|W_{\alpha^c}^i\right|>0\,,\label{blddw}\\
\text{Definition } \ref{def:blockZ}:\;& W_{\alpha^c}^i< 0\text{ for }i\in\alpha\quad \&\quad W_\alpha^i< 0\text{ for }i\in\alpha^c\,.\label{blZw}
\end{align}
The inequalities in (\ref{blddw},\ref{blZw}) are be strict because of the assumption in $\left(\ref{nvi},\ref{nvi2}\right)$.
\end{remark}
The definition of $\alpha$-block diagonal dominance helps us formulate the following simple composition theorem.
\begin{theorem} \label{thm:comp}
Suppose  $\alpha$ \& $\beta$ are stable sets and $\alpha\cap\beta=\emptyset$ then $\gamma=\alpha\cup\beta$ is a stable set iff $W[\gamma]$ is $\alpha$-block diagonally dominant.
\end{theorem}
\begin{proof}
$\Rightarrow$. Since $\alpha$ and $\beta$ are stable sets and $\alpha\cap\beta=\emptyset$, for $i\in\gamma$ we have by the definition in \eqref{fpcond4}
\begin{equation}
s^i_\alpha W_\alpha^i>0\,,\quad s^i_\beta W_\beta^i >0\,.
\label{signab}
\end{equation}
Because of the sign structure in \eqref{signab} we can write the attractor point of $\gamma$ as
\begin{equation}
W_\gamma^i=s^i_\alpha\left|W_\alpha^i\right|+s^i_\beta\left|W_\beta^i\right|\,.
\label{sumg}
\end{equation}
Notice that the set condition for $\gamma$ is satisfied for $i\in\gamma$ if only if \eqref{sumg} is positive. We come to the conclusion that, under the assumption that $\alpha$ and $\beta$ are stable sets, the stable set condition of $\gamma$ for $i\in\gamma$ becomes equivalent to $W[\gamma]$ being an $\alpha$-block diagonal dominant matrix. The only thing remaining to show is that 
\begin{equation}
W_\gamma^i< 0\text{ for } i\not\in\gamma\,.
\end{equation}
This is immediately satisfied by $\alpha$, $\beta$ being stable sets and hence
\begin{equation}
\text{ for } i\not\in\gamma:\;W_\alpha^i,\, W_\beta^i<0\;\Rightarrow\;W_\gamma^i=W_\alpha^i+W_\beta^i<0\,.
\end{equation}
\end{proof}
Theorem (\ref{thm:comp}) states that block diagonal dominance of the principal submatrix $W[\gamma]$ provides the necessary and sufficient structure for two disjoint stable sets to compose a stable superset. Notice that the block dominance condition is in and by itself independent from the stability conditions for the composing subsets. It is an extra constraint on the ``average" strength of the connections between the internally competing states. In the next theorem we'll use the definition of an $\alpha$-block $Z$-matrix to formulate the analogous decomposition theorem.

\begin{theorem}\label{thm:decomp}
Suppose $\gamma$ is a stable set then $\alpha$ \& $\beta$ with $\gamma=\alpha\cap\beta$ and $\alpha\cap\beta=\emptyset$ are stable sets iff $W[\gamma]$ is an $\alpha$-block $Z$-matrix and $W_\alpha^iW_\beta^i>0$ for $i\not\in\gamma$.
\end{theorem}
\begin{proof}
If $\gamma$ is a stable set we have for $i\in\alpha$ that
\begin{equation}
W_\gamma^i=W_\alpha^i+W_\beta^i>0
\end{equation}
and hence
\begin{equation}
W_\alpha^i>-W_\beta^i\,. \label{ineqab}
\end{equation}
Both the definition of stable $\beta$ and $W[\gamma]$ being an $\alpha$-block Z-matrix require that $W_\beta^i<0$ for $i\in\alpha$. In the context of \eqref{ineqab} this inequality immediately implies that $W_\alpha^i>0$ which is the stability condition of $\alpha$ for $i\in\alpha$. By complete analogy for $i\in\beta$, we can conclude that under the assumption of stable $\gamma$, the stable set conditions of $\alpha$ and $\beta$ for $i\in\gamma$ become equivalent to $W[\gamma]$ being an $\alpha$-block $Z$-matrix. We still have to show the equivalence for $i\not\in\gamma$:

$\Rightarrow$. If $\alpha$ and $\beta$ are stable then $W_\alpha^i,\, W_\beta^i<0$. Hence $W_\alpha^i W_\beta^i>0$.

$\Leftarrow$. Take the product
\begin{equation}
W_\alpha^i W_\gamma^i=\left(W_\alpha^i\right)^2+W_\alpha^i W_\beta^i>0\,.
\end{equation}
Since $W_\gamma^i$ is negative, it follows that $W_\alpha^i<0$. Similarly $W_\beta^i<0$.
\end{proof}
\subsection{Coupling} \label{subs:int}
Both the composition and decomposition theorems are special cases of a general state coupling theorem. The theorem is a direct consequence of the Boolean logic that is associated to partial orders. We state it below.\footnote{We express the theorem in terms of the Hadamard product between two vectors. It is defined identical to its operation on matrices
\begin{equation}
\left(S_\alpha\boldsymbol{W_\alpha}\circ S_\alpha\boldsymbol{W_\alpha}\right)^i\equiv\left(S_\alpha\boldsymbol{W_\alpha}\right)^i\left(S_\beta\boldsymbol{W_\beta}\right)^i\quad \forall i\in[n]\,.
\end{equation}
}
\begin{theorem}\label{thm:int}
Suppose $\alpha$ is a stable set then $\beta$ is a stable set iff $S_\alpha\boldsymbol{W_\alpha}\circ S_\beta\boldsymbol{W_\beta}>0$.
\end{theorem}
\begin{proof}
$\Rightarrow$. This follows trivially from the stable set condition.

$\Leftarrow$. Since $\alpha$ is stable, we have $S_\alpha\boldsymbol{W_\alpha}>0$. Therefore we can divide both sides of the inequality $S_\alpha\boldsymbol{W_\alpha}\circ S_\beta\boldsymbol{W_\beta}>0$ by $S_\alpha\boldsymbol{W_\alpha}$ resulting in $S_\beta\boldsymbol{W_\beta}>0$ which is the stable set condition of $\beta$.
\end{proof}
The theorem can be easily extended, e.g. a stable $\alpha$ and $\beta$ will imply a stable $\gamma$ if and only if $S_\alpha\boldsymbol{W_\alpha}\circ S_\beta\boldsymbol{W_\beta}\circ S_\gamma\boldsymbol{W_\gamma}>0$. One can check that the application of Theorem \ref{thm:int} and its generalizations to the cases in Theorems \ref{thm:comp}, \ref{thm:decomp} will provide the conditions as expressed there.
\begin{exmp}
Take a two dimensional network with weight matrix
\begin{equation}
W=
\begin{pmatrix}
a & b \\
c & d 
\end{pmatrix}\,.
\end{equation}
Suppose that the weights are constrained to satisfy
\begin{equation}
S_{\left\{1\right\}}\boldsymbol{W_{\left\{1\right\}}}\circ S_{\left\{2\right\}}\boldsymbol{W_{\left\{2\right\}}}=
-\begin{pmatrix}
ab \\
cd
\end{pmatrix}
>\boldsymbol{0}\,.\label{interc}
\end{equation}
If $\left\{1\right\}$ is a stable set of the network then by the stability condition \eqref{fpcond4} we must have that $a>0$ and $c<0$. Therefore by the constraint in \eqref{interc} we must have that $b<0$ and $d>0$. These are the requirements for \eqref{fpcond4} to be satisfied by the set $\left\{2\right\}$.
\end{exmp}
\section{Conclusion} \label{sec:con}
In this article we studied a rate-based neural circuit model with Heaviside step activation function also known as Glass networks \cite{edwards2000analysis}. The piecewise linearity allows us to predict steady state dynamical behaviour without integrating or carrying out any stability analysis. We showed that the necessary and sufficient conditions for multistability can be formulated in terms of matrix positivity constraints on the synaptic connection weights \eqref{fpcond4}. These constraints have a direct impact on the sign patterns/neuronal configurations the network can possess. We formulated the necessary excitatory/inhibitory connections for several classes of families of stable sets (\cref{thm:nsp,thm:stsg,thm:ns}). The class of sufficient signed graphs for steady state dynamics was also identified (\cref{thm:ssp}). We show that this class is equivalent to the class of patterns that require the absence of stable substates (\cref{thm:msp}). Glass networks with stable steady states obey a uniquely identifying weight matrix factorization theorem (\cref{thm:fac}). We analyzed the conditions under which stable states can compose/decompose into stable macro-/microstates (dis-)allowing pattern completion (\cref{thm:comp,thm:decomp}). We ended with a general state coupling theorem which shows how stable states can become coupled and hereby put constraints on the pattern storage capacities of the network (\cref{thm:int}).\\

Multistability as a property of neural network activity is recognized to be of critical importance in many behavioral and cognitive functions of the nervous system \cite{la2019cortical,brinkman2022metastable}. First and foremost as a mechanism for associative memory encoding and retrieval as modelled by Hopfield \cite{hopfield1982neural}. Although general conditions for stability can be derived, the analysis of multiple stationary solutions in Hopfield-type networks is complicated by the continuous sigmoidal form of the activation functions \cite{cheng2006multistability}. A substitution by hard-switching functions, as in Glass networks, allows for full analytical control, but might come at the cost of biological plausibility. It has been shown however that Glass networks behave very similar to networks with steep sigmoid functions and the dynamical complexity typically gets reduced at lower gain \cite{sompolinsky1988chaos,edwards2000analysis,lewis1991steady,glass1978prediction}.\\

Our results are complementary to the work on multistability in TLNs. For comparison a distinction has to be made between \textit{admissible} and \textit{permitted} sets \cite{curto2020combinatorial}. Sets are called admissible when they contain their attractor point. Sets are called \textit{permitted} when for some input $\boldsymbol{\mu}$ the set contains a stable state. Since one can always find an input for which the attractor point lies within a specific set, whether or not a set is permitted completely depends on the stability of the fixed point. In TLNs, only when a set is admissible and permitted does it contain a stable steady state. This is contrasted by Glass networks where all sets are permitted because of the triviality of the stability analysis. It has been shown that for TLNs with symmetric weight matrix, the condition for the existence of multiple permitted sets can be formulated in terms of a positivity constraint on the weights of the connectivity matrix \cite{hahnloser2000permitted}. As far as we know, it has not yet been shown that such a positivity constraint also exists for determining the presence of admissible sets. By focusing on Glass networks, our work answers this question positively.\\

The semipositivity constraint stems from the relation between transformations preserving polyhedral cones and semipositive maps \cite{tsatsomeros2016geometric}. The discrete geometry of piecewise smooth networks was already explored in other work \cite{curto2020combinatorial,edwards2005matrices,farcot2006geometric}. It is an interesting question whether other dynamical attractors such as limit cycles or chaotic (strange) attractors can also be formulated in terms of simple algebraic constraints on the parameters of the network. Our result on multistability confirms the effectiveness of a geometric approach to this question.\\

Finally, we want to mention that while our work focused on the relation between (clustered) neural architecture and multistable dynamics within the context of a deterministic system, realistic neural networks must be able to function in the presence of stochastic noise \cite{brinkman2022metastable,litwin2012slow}. In networks with many local minima, noise fluctuations can cause the system to escape the basin of attraction of one stable fixed point to another \cite{bressloff2010metastable}. These metastable transitions affect the robustness and sensitivity of a neural system which in turn could provide additional constraints on the structure of the network connections. For example it is not only for the well-functioning of associative memory but also for working memory and decision making crucial to have metastable dwelling times within specific ranges. A non-equilibrium landscape and flux approach \cite{brinkman2022metastable} to Glass networks is therefore of importance. A natural next step would be to study whether also in a stochastic framework Glass networks are able to offer a simple and analytically tractable way to understand the relation between neural function and structure. 

\begin{appendix}
\section{Input/Output Equivalence} \label{app:nioe}
The nonlinear network model 
\begin{equation}
\dot y^i =-y^i + \theta\left( W^i{}_j\, y^j\right)
\label{bmodel}
\end{equation}
differs from \eqref{model} by the position of the activation function on the output generated by the states $y^i$. The phase space of \eqref{bmodel} is hereby constrained to the positive orthant $\mathbb{R}^{n}_+$. However the model in \eqref{model} has the advantage that the parts of the partition are the orthants and are hence independent of the weights. This simplifies the visualizations of the model. It has been known for a long time that both models are dynamically equivalent \cite{feng1996qualitative}. In fact the network constraints associated to multistability are exactly the same. This is easily checked. We first define the parts
\begin{equation}
{\cal P}_\alpha^\prime\equiv\left\{y^i\,|\,\theta\left(W^i{}_j\, y^j\right)=p_\alpha^i\right\}\,.
\label{Pprime}
\end{equation} 
The fixed points of \eqref{bmodel} are the solutions of
\begin{equation}
y^{*}_\alpha{}^i=\theta\left( W^i{}_j\, y^*_\alpha{}^j\right)\,.
\end{equation}
By definition \eqref{Pprime} this translates into
\begin{equation}
p_\alpha^i=\theta\left( W^i{}_j\, p_\alpha^j\right)\,.
\end{equation}
Under the assumption of the constraint in \eqref{nvi}, \eqref{nvi2} the condition for a stable set is
\begin{equation}
S_\alpha W \boldsymbol{p_\alpha}=S_\alpha W P_\alpha\, \boldsymbol{1}>\boldsymbol{0}\,,
\end{equation}
which is the same as \eqref{fpcond4}. All results of Sections \ref{sec:sp} and \ref{sec:fdi} are deduced from the stable set condition. Therefore the theorems therein can directly be applied to model \eqref{bmodel}.
\section{Sign Patterns and Output Constraint} \label{app:bsp}
We didn't control for the constraint \eqref{nvi}, \eqref{nvi2} in the calculation of the number of allowed sign patterns of the network in Corollaries \ref{cor:fs}, \ref{cor:fsd} and \ref{cor:fsn}. The constraint on its own disallows a portion of the parameter space. If included in the analysis one has to make the distinction between the cases of vanishing and nonvanishing input. For nonvanishing input the allowed row signatures are simply those of the theorems but for dimension $n-1$ because of the hyperplane projection. In the case of vanishing input the constraint dissallows all sign pattern that have zero connections. The additional constraining effect of stable states is then only on the positive or negative sign of the connection. We state these results below:
\begin{itemize}
\item Vanishing input:
\begin{equation}
\# \text{signatures}=
\begin{cases}
\text{Cor. \ref{cor:fs}: }& 2^n\big(1-\left(\ft{1}{2}\right)^{k}\big)\\
\text{Cor. \ref{cor:fsd}: }& 2^n\prod_{\mathtt{i}\in \mathtt{I}}\big(1-\left(\ft{1}{2}\right)^{k_\mathtt{i}}\big) \\
\text{Cor. \ref{cor:fsn}: }& 2^{n}\left(1-\left(\ft{1}{2}\right)^{k_\mathtt{1}}-\sum_{\mathtt{i}>\mathtt{1}}\left(\ft{1}{2}\right)^{k_\mathtt{i}}p_{\alpha_{\mathtt{i}}\setminus\alpha_{\mathtt{i-1}}}^i\right)
\end{cases}\,,
\end{equation}
\item Nonvanishing input:
\begin{equation}
\# \text{signatures}=
\begin{cases}
\text{Cor. \ref{cor:fs}:}& 3^{n-1}\big(1-\left(\ft{2}{3}\right)^{k-1}\big)\\
\text{Cor. \ref{cor:fsd}:}& 3^{n-1}\prod_{\mathtt{i}\in \mathtt{I}}\big(1-\left(\ft{2}{3}\right)^{k_\mathtt{i}-1}\big) \\
\text{Cor. \ref{cor:fsn}:}& 3^{n-1}\left(1-\left(\ft{2}{3}\right)^{k_\mathtt{1}-1}-\sum_{\mathtt{i}>\mathtt{1}}\left(1-2^{-k_\mathtt{i}+1}\right)\left(\ft{2}{3}\right)^{k_\mathtt{i}-1}p_{\alpha_{\mathtt{i}}\setminus\alpha_{\mathtt{i-1}}}^i\right)
\end{cases}\,.
\,.
\end{equation}
\end{itemize}
\end{appendix}
\providecommand{\href}[2]{#2}\begingroup\raggedright\endgroup

\end{document}